%% file: main.tex
\documentclass[11pt]{article}

\usepackage{macros}
\usepackage{enumitem}

\bluehyperref
\usebiggeometry
\usepackage[numbers]{natbib}
\usepackage{setspace}
\usepackage{funtitle}

\newcommand{\ppipp}{{\large \texttt{PPI++}}\xspace}
\newcommand{\ppipptitle}{{\fontsize{16pt}{0}\selectfont{\texttt{PPI\raisebox{.0ex}{++}}}\xspace}}

\begin{document}

\title{\ppipptitle: Efficient Prediction-Powered Inference}
\author{Anastasios N. Angelopoulos$^{1}$\footnote{The authors are ordered alphabetically.} \quad  John C. Duchi$^{2\,*}$ \quad  Tijana Zrnic$^{3\,*}$ \vspace{.5em} \\
  \normalsize{$^1$Department of Electrical Engineering and Computer Sciences, UC Berkeley}\\
  \normalsize{$^{2}$Departments of Statistics and Electrical Engineering,
    Stanford University} \\
  \normalsize{$^3$Department of Statistics and Stanford Data Science, Stanford University}
}
\date{}

\maketitle

\begin{abstract}
We present \ppipp: a computationally lightweight methodology for estimation and inference based on a small labeled dataset and a typically much larger dataset of machine-learning predictions. The methods automatically adapt to the quality of available
  predictions, yielding easy-to-compute confidence sets---for parameters
  of any dimensionality---that always improve on classical intervals using
  only the labeled data. \ppipp builds on prediction-powered inference (PPI),
  which targets the same problem setting, improving its computational and statistical efficiency.
Real and synthetic experiments demonstrate
  the benefits of the proposed adaptations.
\end{abstract}


\input{introduction}

\input{glms}

\input{m-estimators}

\input{equivalent-confidence-sets}

\input{efficiency}

\input{experiments}

\bibliography{bib}
\bibliographystyle{abbrvnat}

\newpage

\appendix
\input{proofs}

\input{one-step-experiments}

\input{odds-ratio-calculations}

\end{document}

%% file: introduction.tex

\section{Introduction}

We cannot realize the promise of machine learning models---their ability to
ingest huge amounts of data, predict accurately in many domains, their
minimal modeling assumptions---until we can effectively, efficiently, and
responsibly use their predictions in support of scientific inquiry. This paper
takes a step toward that promise by showing how, in problems where
labeled observations or precise measurements of a phenomenon of interest are
scarce, we can leverage predictions from machine-learning algorithms to
supplement smaller real data to validly increase inferential power.

There is no reason to trust black-box machine-learned predictions in place
of actual measurements, calling into question the validity of studies
naively using such predictions in data analyses.
\citeauthor{AngelopoulosBaFaJoZr23}'s
\emph{Prediction Powered Inference} (PPI)~\citep{AngelopoulosBaFaJoZr23}
tackles precisely this issue, providing a framework for valid statistical
inference---computing p-values and confidence intervals---even using
black-box machine-learned models.
The idea of PPI is to combine a small amount of labeled data with a larger
collection of predictions to obtain valid and small
confidence intervals.  While the framework applies to a broad range of
estimands and arbitrary machine-learning models, the original
proposal has limitations:
\begin{enumerate}[label=(L\arabic*)]
\item \label{item:intractable}
  for many estimands, such as regression coefficients,
  computing the intervals is intractable, especially in high dimensions;
\item \label{item:adaptivity}
  when the provided predictions are inaccurate, the intervals can
  be worse than the ``classical intervals'' using only the labeled
  data.
\end{enumerate}
\ppipp addresses both limitations.

To set the stage, for observations $(X, Y) \sim \P$ and a loss function
$\loss_\theta(x, y)$, we wish to estimate and perform
inference for the
$d$-dimensional parameter
\begin{equation}
  \label{eq:m-est}
  \theta\opt = \argmin_{\theta \in \R^d}
  \{L(\theta) \defeq \E[\loss_\theta(X, Y)]\}.
\end{equation}
Typical loss functions include the squared error for linear
regression, $\loss_\theta(x, y) = \half(x^\top \theta - y)^2$, or the
negative log-likelihood of a conditional model for $Y$ given $X$,
$\loss_\theta(x, y) = -\log p_\theta(y \mid x)$.  We receive $n$
labeled data points $(X_i, Y_i) \simiid \P$ and $N$ unlabeled data points
$\Xt_i \simiid \Px$, where $X_i$ and $\Xt_i$ have identical distributions.
The key assumption is that we also have access to a black-box model $f$
mapping $X$ to predictions $\what{Y} = f(X)$, which we use to impute
labels for $\Xt_i$. PPI~\cite{AngelopoulosBaFaJoZr23} performs inference on a so-called \emph{rectifier} (an analog of a control variate or debiasing
term) to incorporate these black-box predictions; we reconsider
their methodology to provide two main improvements, adding the
\texttt{\raisebox{.0ex}{++}}.

\paragraph{\texttt{\raisebox{.0ex}{+}}Computational efficiency.}
The original prediction-powered inference constructs a confidence set
for $\theta\opt$ by searching through all possible values of
the estimand $\theta$ and performing a hypothesis test for each.
Some estimands, such as means, admit simplifications, but others, such as
logistic regression or other generalized linear model (GLM) coefficients, do
not. When it does not simplify, the original PPI
method~\cite{AngelopoulosBaFaJoZr23} grids the space of $\theta$ and tests
each potential parameter separately; this is both inexact and intractable,
especially in more than two dimensions. 

To address problem~\ref{item:intractable}, we develop computationally fast
convex-optimization-based algorithms for computing prediction-powered point
estimates and intervals for a wide class of estimands, which includes
all GLMs.  The algorithms use the asymptotic normality of these point
estimates and are thus asymptotically exact, computationally lightweight,
and easy to implement.  We show that the resulting intervals are
asymptotically equivalent to an idealized version of the intervals obtained
via the testing strategy (with an infinite grid).  In practice, the
intervals are often tighter, especially in high dimensions.

An important consequence of this theory appears when we wish to estimate
only a single coordinate $\theta\opt_j$ in the model, such as a causal
parameter, while controlling for other covariates and adjusting for nuisance
parameters. Whereas the original PPI~\cite{AngelopoulosBaFaJoZr23} implicitly requires a multiplicity
correction over the entire parameter vector, \ppipp's improved statistical efficiency in the presence
of nuisances allows for inference on single coordinates of regression
coefficients while avoiding multiplicity correction.

\paragraph{\texttt{\raisebox{.0ex}{+}}Power tuning.}
To address problem~\ref{item:adaptivity}, we develop a lightweight
modification of PPI that adapts to the accuracy of the predictive model $f$.
The modification automatically selects a parameter $\lambda$ to interpolate
between classical and prediction-powered inference.  The optimal value of
this $\lambda$ admits a simple plug-in estimate, which we term \emph{power
tuning}, because it optimizes statistical power.  Our experiments and theoretical results indicate that power tuning
is essentially never worse than either classical or prediction-powered
inference, and it can outperform both substantially. Power tuning improves the
performance of the original PPI algorithms when $f$ is not
highly accurate, and, when $f$ is accurate, it can sometimes
guarantee full statistical efficiency.

\section{Preliminaries}

We formalize our problem setting while recapitulating the
prediction-powered-inference (PPI) approach~\cite{AngelopoulosBaFaJoZr23}.
Recall that we have $n$ labeled data points, $(X_i,Y_i) \simiid \P,
i\in[n]$, as well as $N$ unlabeled data points, $\Xt_i \simiid \Px,
i\in[N]$, where the feature distribution is identical across both
datasets, $\P = \Px \times \P_{Y \mid X}$.  In addition to the data, we have
access to a machine-learning model $f$ that predicts outcomes from
features.

Define the population losses
\begin{equation*}
  L(\theta) \defeq \E[\loss_\theta(X, Y)]
  ~~ \mbox{and} ~~
  L^f(\theta) \defeq \E[\loss_\theta(X, f(X))].
\end{equation*}
The starting point of PPI and our approach is the recognition that $\E[\loss_\theta(X, f(X))]
= \E[\loss_\theta(\wt{X}, f(\wt{X}))] = L^f(\theta)$, so that the ``rectified'' loss,
\begin{equation*}
  \LPP(\theta) \defeq L_n(\theta) +
  \widetilde L^f_N(\theta) - L_n^f(\theta),
\end{equation*}
where
\begin{equation*}
 L_n(\theta) \defeq \frac{1}{n} \sum_{i = 1}^n \loss_\theta(X_i, Y_i), ~~ L_n^f(\theta) \defeq \frac{1}{n} \sum_{i = 1}^n \loss_\theta(X_i, f(X_i))
  ~~ \mbox{and} ~~
  \wt{L}_N^f(\theta) \defeq \frac{1}{N} \sum_{i = 1}^N \loss_\theta(\wt{X}_i,
  f(\wt{X}_i)),
\end{equation*}
is unbiased for the true objective: $\E[\LPP(\theta)] = L(\theta)$.

In the original, PPI first forms
$(1-\alpha)$-confidence sets for the gradient of the loss $\nabla L(\theta)$, denoted $\C^{\nabla}_\alpha(\theta)$, 
and then constructs the confidence set $\CPP_\alpha$ for $\theta\opt$ as
\begin{equation}
  \label{eq:CPP-original}
  \CPP_\alpha = \{\theta \in \R^d ~ \mbox{s.t.} ~ \zeros
  \in \C^\nabla_{\alpha}(\theta)\}.
\end{equation}
Constructing $C^\nabla_\alpha(\theta)$ for a \emph{fixed} $\theta$ is a trivial application of the central limit theorem to $\nabla \LPP (\theta)$. 
The issue is that forming the confidence set for \emph{every} $\theta \in \R^d$ is computationally infeasible, because one cannot simply iterate through every possible value.

\subsection{Computational efficiency: overview}

We take PPI as our point of departure, but adopt a different perspective. In particular, the
rectified loss leads to the \emph{prediction-powered point estimate},
\begin{align}
  \label{eq:ppi}
  \tag{$P_1$}
  \thetaPP &=\argmin_\theta \LPP(\theta), \text{ where } \LPP(\theta) \defeq L_n(\theta) + \widetilde L^f_N(\theta)  - L_n^f(\theta).
\end{align}
This estimate (along with its forthcoming variant~\eqref{eq:ppi-lambda}) is the
main object of study in this paper, and we derive its asymptotic normality
about $\theta\opt$.  This result is the keystone to our efficient
\ppipp algorithms, allowing us to construct confidence intervals of the form
\begin{equation}
  \label{eq:CPP-new}
  \CPP_{\alpha} = \left\{\thetaPP_j \pm z_{1-\alpha/2} \cdot \hat\sigma_j / \sqrt{n} \right\},
\end{equation}
where $\hat\sigma_j^2$ is a consistent estimate of the
asymptotic variance of $\thetaPP_j$ and $z_{q}$ denotes the
$q$-quantile of the standard normal distribution. Similar confidence intervals centered around $\thetaPP$ apply when inferring the whole vector $\theta\opt$ rather than a single coordinate $\theta\opt_j$.

It should be clear that the set~\eqref{eq:CPP-new} is easier to compute than the set~\eqref{eq:CPP-original}.
The latter requires the analyst to ``iterate through all $\theta \in \R^d$'' and test
optimality separately; the former has no such requirement, as \ppipp
simply requires placing error bars around a point estimate.  
This perspective, though simple, results in far more efficient algorithms.
In Section~\ref{sec:ppi++_glms}, we provide algorithms for generalized linear models, and in Section~\ref{sec:ppipp-m-estimation}, we show the case of general M-estimators. 
In Section~\ref{sec:equivalence}, we show that the \ppipp intervals are asymptotically equivalent to the naive strategy~\eqref{eq:CPP-original}, so the increased computational efficiency comes with no cost in terms of statistical power.

\subsection{Power tuning: overview}

The prediction-powered estimator~\eqref{eq:ppi} need not outperform classical inference when $f$ is inaccuate:
limitation~\ref{item:adaptivity} of PPI. 
We thus generalize PPI to allow it to adapt to the ``usefulness'' of the supplied predictions, yielding an estimator never worse than classical inference. 
To do so, we introduce a tuning parameter $\lambda \in \R$, defining
\begin{equation}
\label{eq:ppi-lambda}
\tag{$P_\lambda$}
\thetaPP_\lambda = \argmin_\theta \LPP_\lambda(\theta), \text{ where } \LPP_\lambda(\theta)\defeq L_n(\theta) + \lambda \cdot (\widetilde L^f_N(\theta) - L_n^f(\theta)).
\end{equation}
Certainly the objective remains unbiased for any $\lambda$, as
$\E[\LPP_\lambda(\theta)] = L(\theta)$. The problem \eqref{eq:ppi-lambda}
recovers \eqref{eq:ppi} when $\lambda = 1$ and reduces to performing
classical inference, ignoring the predictions, when
$\lambda=0$. 
In Section~\ref{sec:setting-lambda}, we show how to optimally tune $\lambda$ to maximize estimation and inferential efficiency, yielding a data-dependent tuning parameter $\hat{\lambda}$.
This adaptively chosen $\hat{\lambda}$ leads to better estimation and inference than both the classical and PPI strategies.
(In fact, sometimes it is fully statistically efficient, though that is
not our focus here.)
Section~\ref{sec:experiments} puts everything together through a series of experiments highlighting the benefits of \ppipp.

\subsection{Related work}

\citet{AngelopoulosBaFaJoZr23} develop prediction-powered inference (PPI)
to allow estimation and inference using black-box machine learning models;
our work's debt to the original proposal and motivation is clear.
%
The PPI observation model---where we receive $n$ labeled observations and
$N$ unlabeled observations---also centrally motivates semi-supervised
learning and inference~\cite{ZhangBrCa19, ZhangBr22, SongLiZh23}.
This line of work also considers the problem of inference with a labeled
dataset and an unlabeled dataset; the key point of departure is PPI's access
to a pretrained machine learning model $f$.
Fitting empirically strong predictive models of a target $Y$ given $X$ has
been and remains a major focus of the machine learning
literature~\cite{LeCunBeHi15}, so that leveraging these
models as black boxes, in spite of the lack of essentially \emph{any}
statistical control over them (or even knowledge of how they were fit),
provides a major opportunity to leverage modern prediction techniques.
This enables PPI to apply to estimation problems where, for example,
high-performance deep learning models exist to impute labels (such as
AlphaFold~\cite{JumperEvEtAl21} for biological structures or
CLIP~\cite{RadfordKiHaRaGoAgSaAsMiClKrSu21} for image labeling), and, though we may have no
hope of understanding precisely what the models are doing, to obtain more
powerful inferences in those settings.

\newcommand{\thetaAIPW}{\hat{\theta}^{\textup{AIPW}}}

One may also take a perspective on the \ppipp approach coming from the
literatures on semiparametric inference, causality, and missing
data~\cite{RobinsRo95, BickelKlRiWe98, RobinsRoZh94, Tsiatis06,
  ChernozhukovChDeDuHaNeRo16}, as well as that on survey
sampling~\cite{SarndalSwWr03}.  We only briefly develop the connections, but
it is instructive to formulate our problem as one with missing data, where
the dataset $\{\Xt_i\}_{i=1}^N$ has no labels. Consider now augmented
inverse propensity weighting (AIPW) estimators~\cite{RobinsRo95, Tsiatis06},
focusing in particular on estimating the mean $\E[Y]$. Then because we know
the ratio $r = n/N$ (the propensity score, or probability of missingness), AIPW~\cite{RobinsRo95, Tsiatis06}
estimates $\E[Y]$ by
\begin{align}
  \thetaAIPW & \defeq \frac{1}{n} \sum_{i = 1}^n
  (Y_i - \hat{\E}[Y \mid X = X_i])
  + \frac{1}{N + n} \bigg(\sum_{i = 1}^n \hat{\E}[Y \mid X = X_i]
  + \sum_{i = 1}^N \hat{\E}[Y \mid X = \wt{X}_i]\bigg) \nonumber \\
  & \, = \frac{1}{n} \sum_{i = 1}^n
  (Y_i - f(X_i))
  + \frac{1}{N + n} \bigg(\sum_{i = 1}^n f(X_i)
  + \sum_{i = 1}^N f(\wt{X}_i)\bigg)   \label{eqn:aipw}
  \\
  & \, = \frac{1}{n} \sum_{i = 1}^n
  Y_i 
  + \frac{1}{1 + \frac n N} \bigg(-\frac 1 n \sum_{i = 1}^n f(X_i)
  + \frac 1 N \sum_{i = 1}^N f(\wt{X}_i)\bigg),
  \nonumber
\end{align}
where in the first line $\hat{\E}$ denotes the estimated conditional
expectation of $Y$ given $X$, and in the second we replace it with the
black-box predictor $f$.
This corresponds to the particular choice $\lambda = \frac{1}{1 + r}$ in the
\ppipp estimator~\eqref{eq:ppi-lambda}.
The point of emphasis in PPI is distinct, however; while semiparametric
inference in principle allows machine learning to estimate nuisance
parameters ($\E[Y \mid X]$ in the AIPW estimator~\eqref{eqn:aipw}), we
develop easy-to-compute, practical, and assumption-light estimators, where
we truly treat $f$ as an unknown black box, and \ppipp leverages the
predictions as much as is possible.
The practicality is essential (as our experiments highlight), and it
reflects the growing trend in science to use the strength of modern
machine-learned predictors $f$ to impute labels $Y$.
The debiasing~\eqref{eq:ppi-lambda} thus directly allows arbitrary
imputation of missing labels, avoiding challenges in more common imputation
approaches~\cite{Rubin18}.
We remark in passing that survey sampling approaches frequently use
auxiliary data to improve estimator efficiency~\cite{SarndalSwWr03}, but
these techniques are limited to simple estimands and do not use machine
learning.

Our work can be seen as an improvement of PPI along the lines of Veridical Data Science~\cite{Yu18, YuKumbier20, YuBa24}, the main tenets of which are predictability, computability, and stability (PCS).
On the note of predictability, PPI can diminish statistical power if the model $f$ cannot adequately predict $Y$; our procedure will automatically check if this is the case and reduces to classical inference if the model is bad. Therefore, if classical inference is believed to yield statistically valid results, \ppipp will be valid as well.
On the note of computability, for algorithms like logistic regression, PPI intervals were not efficiently computable, while \ppipp intervals are.

See also \citet{AngelopoulosBaFaJoZr23} for more discussion of related work.


%% file: glms.tex

\section{\ppipptitle~for generalized linear models}
\label{sec:ppi++_glms}

Generalized linear models (GLMs) make the benefits and approach of \ppipp
clearest.  In the basic GLM case, the loss function takes the form
\begin{equation}
    \label{eq:glm-loss}
    \loss_\theta(x,y) = -\log p_\theta(y \mid x)
    = -y\theta^\top x + \psi(x^\top\theta),
\end{equation}
where $p_\theta(y \mid x) = \exp(y x^\top \theta - \psi(x^\top \theta))$ is
the probability density of $y$ given $x$ and the log-partition function
$\psi$ is convex and infinitely differentiable~\cite{Brown86}. 
The setting~\eqref{eq:glm-loss} captures the familiar linear and logistic regression, where in the former case, we have $y \in \R$ and set $\psi(s) = \half s^2$, and in the latter, we have $y \in \{0, 1\}$ and $\psi(s) = \log (1+ e^{s})$.

For any loss of the form~\eqref{eq:glm-loss}, the key to computational
efficiency is that $\LPP_{\lambda}(\theta)$ is convex for $\lambda \in [0, 1]$.
Indeed,  $\lambda \wt{L}_N^f(\theta)$ is convex in $\theta$,
while
\begin{equation*}
  L_n(\theta) - \lambda L_n^f(\theta)
  = - \frac{1}{n}
  \sum_{i=1}^n (Y_i - \lambda f(X_i))\theta^\top X_i
  +  (1-\lambda) \frac{1}{n} \sum_{i=1}^n\psi(X_i^\top \theta)
\end{equation*}
consists of linear and convex components. 
Thus, $\LPP_{\lambda}(\theta)$ is a sum of convex terms, and one can compute $\thetaPP_{\lambda}$ efficiently~\cite{BoydVa04}.
In this setting, the following corollary of our forthcoming
Theorem~\ref{theorem:normality} shows the asymptotic normality of the prediction-powered point estimate.

\begin{corollary}
  \label{cor:normality_glms}
  Assume that $\hat\lambda = \lambda + o_P(1)$ for some limit $\lambda$,
  $\frac{n}{N} \rightarrow r$ and that
  $H_{\theta\opt}\defeq\E[\psi''(X^\top\theta\opt) X X^\top]$ is
  nonsingular. Define the covariance matrices
  \begin{align*}
    V_{f,\theta\opt}^\lambda &\defeq \lambda^2 \cov\left( (\psi'(X^\top\theta\opt)-f(X)) X \right)\text{ and }\\
    V_{\Delta,\theta\opt}^\lambda  &\defeq \cov\left( (1-\lambda)\psi'(X^\top\theta\opt)X
    + ( \lambda f(X) - Y) X\right).
  \end{align*}
  Then for $\Sigma^{\lambda} \defeq H_{\theta\opt}^{-1} \left(r V_{f,\theta\opt}^\lambda  + V_{\Delta,\theta\opt}^\lambda \right) H_{\theta\opt}^{-1}$, we have the convergence
  \begin{equation*}
    \sqrt{n}(\thetaPP_{\hat\lambda} - \theta\opt) \cd \normal\left(0,
    \Sigma^{\lambda} \right).
  \end{equation*}
\end{corollary}
\noindent Corollary \ref{cor:normality_glms} allows the construction of confidence intervals around $\thetaPP_{\hat\lambda}$; see Algorithm \ref{alg:glms}.

We develop results on the optimal choice of $\hat \lambda$ in Section
\ref{sec:setting-lambda}. For now, a reasonable heuristic is to simply think
of $\hat\lambda = 1$ when the predictive model $f$ is accurate and
$\hat\lambda = 0$ when it is not.  When the model $f$ is good and
$\hat\lambda = 1$, we expect $V_{\Delta,\theta\opt}^\lambda = \cov((f(X) -
Y)X) \approx 0$, so that the asymptotic covariance $\Sigma^{\lambda}$ is
proportional to $r = \frac{n}{N}$; a large unlabeled data size $N$ makes
this covariance tiny.
Otherwise, when the model $f$ is bad and $\hat\lambda = 0$, the covariance
reduces to the classical~\cite{VanDerVaart98, Efron22} sandwich-type GLM
covariance
\begin{equation*}
  H_{\theta\opt}^{-1} \cov(\nabla \loss_{\theta\opt}(X, Y)) H_{\theta\opt}^{-1}.
\end{equation*}

Building on Corollary \ref{cor:normality_glms}, we construct confidence
intervals for $\theta\opt$ using a plug-in estimator for
$\Sigma^{\lambda}$. For concreteness, we provide the general
confidence interval construction for GLMs in Algorithm~\ref{alg:glms};
instatiations for particular choices of the loss~\eqref{eq:glm-loss}
follow. In the algorithm, for any map $g$, we use $\widehat\Cov_n(g(X_i,Y_i))$
to denote the empirical covariance of $g(X_i,Y_i)$ computed on the labeled
data, and $\widehat\Cov_{N+n}(g(X_i))$ to denote the empirical covariance of
$g(X_i)$ computed on the labeled and unlabeled data combined.
The guarantees of Algorithm \ref{alg:glms} follow from
Corollary~\ref{cor:normality_glms}.

\begin{algorithm}[t]
\setstretch{1.3}
\caption{\ppipp: Prediction-powered inference for GLMs}
\label{alg:glms}
\begin{algorithmic}[1]
  \Require labeled data $\{(X_i, Y_i)\}_{i=1}^n$, unlabeled data $\{\Xt_i\}_{i=1}^N$, predictive model $f$, error level $\alpha\in~(0,1)$, coefficient index
  $j \in [d]$ of interest
  \State Select tuning parameter $\hat\lambda$
  \Comment{set tuning parameter}

  \State $\thetaPP_{\hat\lambda} = \argmin_\theta \LPP_{\hat\lambda}(\theta)$ \Comment{prediction-powered point estimate}
    \State $\widehat H = \frac{1}{N+n} (\sum_{i=1}^n \psi''(X_i^\top \thetaPP_{\hat\lambda}) X_iX_i^\top + \sum_{i=1}^N \psi''(\Xt_i^\top \thetaPP_{\hat\lambda}) \Xt_i \Xt_i^\top)$
\State $\widehat V_f = \hat\lambda^2 \cdot \widehat\Cov_{N+n} (  (\psi'(X_i^\top \thetaPP_{\hat\lambda})-f(X_i)) X_i)$ 
\State $\widehat V_\Delta = \widehat\Cov_{n} ( (1-\hat \lambda)\psi'(X_i^\top\thetaPP_{\hat\lambda}) + (\hat\lambda f(X_i) -Y_i)X_i)$ 
\State $\widehat\Sigma = \widehat H^{-1}( \frac{n}{N} \widehat V_f + \widehat V_\Delta) \widehat H^{-1}$ \Comment{covariance estimate}
\Ensure 
prediction-powered confidence interval $\CPP_\alpha = \left(\thetaPP_{\hat\lambda, j} \pm z_{1-\alpha/2} \sqrt{\frac{\widehat\Sigma_{jj}}{n}} \right)$
for coordinate $j$
\end{algorithmic}
\end{algorithm}

\begin{corollary}
  Let the conditions of Corollary~\ref{cor:normality_glms} hold.  Then, for
  each coordinate $j \in [d]$, the prediction-powered confidence interval in
  Algorithm \ref{alg:glms} has valid coverage:
  \[\lim_{n} \P\left(\theta\opt_{j} \in \CPP_\alpha \right) = 1-\alpha.\]
\end{corollary}

Algorithm~\ref{alg:glms}'s efficient implementation separates
it from the original PPI methodology~\cite{AngelopoulosBaFaJoZr23}, which in general
requires performing a hypothesis test for each possible $\theta$. For
example, their algorithm for logistic regression discretizes the parameter space
and performs a hypothesis test for \emph{each} grid point, which is
essentially non-implementable in more than two
or three dimensions and prone to numerical inaccuracies.

\begin{example}[Linear regression]
  To obtain confidence intervals on the $j$th coefficient of a
  $d$-dimensional linear regression, run Algorithm~\ref{alg:general_ppi}
  with $\psi(s) = \frac{1}{2}s^2$, so $\psi'(s) = s$ and
  $\psi''(s) \equiv 1$.
  The Hessian  simplifies to $\widehat H = \frac{1}{N+n} (
  \sum_{i=1}^n X_i X_i^\top + \sum_{i=1}^N \tX_i\tX_i^\top )$.
\end{example}

\begin{example}[Logistic regression]
To obtain confidence intervals on the $j$th coefficient of a $d$-dimensional
logistic regression, run Algorithm~\ref{alg:general_ppi} with $\psi(s) =
\log(1 + e^{s})$.  Here, one substitutes $\psi'(s) =
\frac{e^s}{1+e^s}$ and $\psi''(s) = \psi'(s)(1-\psi'(s))
= \frac{e^s}{(1 + e^s)^2}$.
\end{example}

\begin{example}[Poisson regression]
  A variable $Y \sim \poisson(\lambda)$ has
  probability mass $p(y) = \frac{\lambda^y e^{-\lambda}}{y!}$;
  the choice $\lambda = e^{x^\top \theta}$ gives GLM $p_\theta(y \mid x)
  = \exp(y x^\top \theta - e^{-x^\top \theta} - \log(y!))$. Thus
  one runs Algorithm~\ref{alg:general_ppi} with
  $\psi(s) = e^{-s}$, $\psi'(s) = - e^{-s}$, and
  $\psi''(s) = e^{-s}$.
\end{example}

\subsection{General exponential-family regression}
\label{sec:general-expfam}

The results of Section \ref{sec:ppi++_glms} immediately extend to
exponential-family regression models with a general vector-valued sufficient
statistic $T(x,y)$.  In this case, the conditional likelihood becomes
\begin{equation*}
  p_{\theta}(y \mid x) = \exp\left( T(x,y)^\top \theta - \psi(\theta, x) \right),
\end{equation*}
where $\psi(\theta,x) = \log \int \exp(T(x,y)^\top \theta) d\nu(y \mid x)$
for some carrier measure $\nu$ on $Y$.
The log-partition $\psi(\theta, x)$ is $\mc{C}^\infty$
and convex in $\theta$~\cite{Brown86}, and it 
induces the convex loss
\begin{equation}
  \label{eq:exponential-family-loss}
  \loss_\theta(x,y) = -T(x,y)^\top \theta + \psi(\theta,x).
\end{equation}
Taking $T(x,y)=y x$ recovers the GLM losses~\eqref{eq:glm-loss}.

\begin{example}[Multiclass logistic regression]
  In $k$-class multiclass logistic regression with covariates $x \in \R^d$,
  the parameter vector $\theta$ is a $kd$-dimensional vector whose $y$th
  block $\theta_y \in \R^d$ of components represent parameters for class $y
  \in [k]$. The sufficient statistic is $T(x,y) = E_y x$, where $E_y$ is the
  matrix $(\mathbf{0}_{d\times (y-1)d}, \mathbf{I}_{d \times d},
  \mathbf{0}_{d\times (K-y)d})^\top$, and
  \begin{equation*}
    \loss_\theta(x, y)
    = -x^\top \theta_y + \log\bigg(\sum_{y' = 1}^k \exp(x^\top \theta_{y'})\bigg).
  \end{equation*}
  Thus $\psi(\theta, x) = \log(\sum_{y' = 1}^k e^{x^\top \theta_{y'}})$ is the
  log-partition function.
\end{example}

The linear plus convex representation of the exponential-family
loss~\eqref{eq:exponential-family-loss} means
it too admits the same computational efficiencies
as the basic GLM approach.

%% file: m-estimators.tex

\newcommand{\lipobj}{M}

\section{\ppipptitle~for general M-estimators}
\label{sec:ppipp-m-estimation}

Our results on generalized linear models follow from more general
results on M-estimators, which we turn to here. We consider convex
losses $\loss_\theta(x, y)$, and we continue to write
\begin{equation*}
  \theta\opt \defeq \argmin_\theta \E [\loss_\theta(X,Y)]
\end{equation*}
as the target of interest.
Our main theorem applies to what we term \emph{smooth enough} losses, which
roughly satisfy that $\loss_\theta$ is differentiable at $\theta\opt$ with
probability 1 and is locally Lipschitz; see
Definition~\ref{definition:smooth-enough} in
Appendix~\ref{sec:proof-normality}.
%
%
Smooth enough losses admit the following general asymptotic normality
result, whose proof we defer to Appendix~\ref{sec:proof-normality}. We use
the short-hand notation $\nabla\ell_{\theta} = \nabla\ell_{\theta}(X,Y)$ and
$\nabla \ell_{\theta}^f = \nabla \ell_{\theta}(X,f(X))$.

\begin{theorem}
  \label{theorem:normality}
  Assume that $\hat\lambda = \lambda + o_P(1)$ and that
  $\frac{n}{N}\rightarrow r$ for some $r\geq 0$. Let the losses
  $\loss_\theta$ be
  smooth enough (Definition~\ref{definition:smooth-enough}) and let $H_{\theta\opt} \defeq \nabla^2 L(\theta\opt)$.
  Define the covariance matrices
  \begin{align*}
    V_{f,\theta\opt}^\lambda \defeq \lambda^2 \cov(\nabla \loss_{\theta\opt}^f)\text{ and } V_{\Delta,\theta\opt}^\lambda  \defeq \cov( \nabla \loss_{\theta\opt} -  \lambda \nabla \loss_{\theta\opt}^f ).
  \end{align*}
  If the
  consistency $\thetaPP_{\hat \lambda} \stackrel{p}{\to} \theta\opt$ holds,
  then
  \begin{equation*}
    \sqrt{n}(\thetaPP_{\hat \lambda} - \theta\opt) \cd \normal\left(0, \Sigma^\lambda\right),
  \end{equation*}
  where 
  \begin{equation}
  \label{eqn:sigma-lambda}	
\Sigma^\lambda \defeq H_{\theta\opt}^{-1} \left(r \cdot
    V_{f,\theta\opt}^\lambda +
    V_{\Delta,\theta\opt}^\lambda
    \right) H_{\theta\opt}^{-1}.
    \end{equation}
\end{theorem}

The conditions sufficient for Theorem \ref{theorem:normality} to hold are
standard and fairly mild. As is familiar from classical asymptotic theory,
the one that is least straightforward to verify is the consistency of
$\thetaPP_{\hat \lambda}$, so here we give a few sufficient conditions. At a
high level, if $\LPP_\lambda$ is convex---as in the case of mean estimation
or GLMs---or the domain $\Theta$ of $\theta$ is compact, consistency holds
under mild additional regularity. Below we state the guarantee of
consistency for convex $\LPP_\lambda$ (providing a proof in
Appendix~\ref{sec:proof-consistency-convex}); for the treatment when
$\Theta$ is compact, see \citet[Chapter 5]{VanDerVaart98}.

\begin{proposition}
  \label{prop:consistency-convex}
  Assume that $\LPP_\lambda(\theta)$ is convex in $\theta$ with probability
  $1$ and that $\theta\opt$ is a unique minimum of $L(\theta)$.  Suppose
  also that $\loss_{\theta}$ is locally Lipschitz near $\theta\opt$
  (in the sense of
  Definition~\ref{definition:smooth-enough}\ref{item:locally-lipschitz}).
  Then, $\thetaPP_{\lambda} \cp
  \theta\opt$. If $\hat{\lambda} \cp \lambda$, then
  so long as $\P(\LPP_{\hat{\lambda}} ~\mbox{is~convex}) \to 1$, the same
  result holds.
\end{proposition}
\noindent
Because GLMs~\eqref{eq:glm-loss} and~\eqref{eq:exponential-family-loss}
have a convex loss $\LPP_\lambda(\theta)$
for $\lambda \in [0, 1]$, Theorem~\ref{theorem:normality} and
Proposition~\ref{prop:consistency-convex} combine to imply the asymptotic
normality result of Corollary~\ref{cor:normality_glms}.

With Theorem \ref{theorem:normality} in hand, we can generalize \ppipp developed for GLMs to general convex
M-estimators. In Algorithm~\ref{alg:general_ppi} we state the main
procedure, and Corollary~\ref{cor:validity_general} gives its validity.
Analogously to last section, $\widehat\Cov_n(\nabla
\loss_{\theta}),\widehat\Cov_{N+n}(\nabla \loss^f_{\theta})$, and
$\widehat\Cov_n(\nabla \loss_{\theta} - \hat \lambda \nabla
\loss_{\theta}^f)$ denote empirical covariance matrices computed using the
labeled data or the labeled and unlabeled data combined.

\begin{algorithm}[t]
\setstretch{1.3}
\caption{\ppipp: Prediction-powered inference for general M-estimators}
\label{alg:general_ppi}
\begin{algorithmic}[1]
  \Require labeled data $\{(X_i, Y_i)\}_{i=1}^n$, unlabeled data
  $\{\Xt_i\}_{i=1}^N$, predictive model $f$, error level $\alpha\in~(0,1)$,
  coefficient $j\in[d]$ of interest
  \State Select tuning parameter $\hat\lambda$
  \Comment{set tuning parameter}
  \State $\thetaPP_{\hat\lambda} = \argmin_\theta \LPP_{\hat\lambda}(\theta)$
  \Comment{prediction-powered point estimate}
  \State $\widehat H = \frac{1}{n} \sum_{i=1}^n \nabla^2 \ell_{\thetaPP_{\hat\lambda}}(X_i, Y_i)$
  \State $\widehat V_f = \hat \lambda^2 \widehat\Cov_{N+n} (\nabla \ell_{\thetaPP_{\hat\lambda}}^f)$ 
  \State $\widehat V_\Delta = \widehat\Cov_{n} (\nabla \ell_{\thetaPP_{\hat\lambda}} - \hat\lambda \nabla \ell_{\thetaPP_{\hat\lambda}}^f)$ 
  \State $\widehat\Sigma = \widehat H^{-1}( \frac{n}{N} \widehat V_f + \widehat V_\Delta) \widehat H^{-1}$ \Comment{covariance estimate}
  \Ensure 
  prediction-powered confidence set $\CPP_\alpha = \left(\thetaPP_{\hat\lambda, j} \pm z_{1-\alpha/2} \sqrt{\frac{\widehat\Sigma_{jj}}{n}} \right)$
\end{algorithmic}
\end{algorithm}

\begin{corollary}
  \label{cor:validity_general}
  Let the assumptions of Theorem~\ref{theorem:normality} hold and
  assume $\nabla^2 \loss_\theta$ is continuous in $\theta$.
  Then the
  prediction-powered confidence interval in Algorithm \ref{alg:general_ppi}
  has valid coverage: for any $j \in [d]$,
  \begin{equation*}
    \lim_{n} \P\left(\theta\opt_j \in \CPP_\alpha \right) = 1-\alpha.
  \end{equation*}
\end{corollary}

If we seek a confidence set for the whole vector $\theta\opt$, rather than a
single coordinate, many choices are
possible for a resulting confidence set. The ellipse
$\mc{E}_{\Sigma,\alpha} = \{v \in \R^d \mid v^\top \Sigma^{-1} v \le
\chi^2_{d,1-\alpha}\}$, where $\chi^2_{d,1-\alpha}$ denotes the $(1-\alpha)$-quantile
of the $\chi^2$ distribution with $d$ degrees of freedom,
satisfies $\P(Z \in \mc{E}_{\Sigma,\alpha}) = 1 - \alpha$ for
$Z \sim \normal(0, \Sigma)$. Thus
we can modify Algorithm~\ref{alg:general_ppi} to return
\begin{equation*}
  \CPP_\alpha = \thetaPP_{\hat\lambda}
  + \mc{E}_{\what{\Sigma} / n, \alpha},
\end{equation*}
which yields $\lim_n \P(\theta\opt \in \CPP_\alpha) = 1 - \alpha$.
Alternatively, we can apply the union bound, which gives rectangular sets of the form
\begin{equation*}
  \CPP_\alpha = \left[\thetaPP_{\hat\lambda,j} \pm \frac{z_{1-\alpha/(2d)}}{\sqrt{n}}
    \sqrt{\what{\Sigma}_{jj}} \right]_{j = 1}^d.
\end{equation*}
These sets guarantee $\liminf_n \P(\theta\opt \in \CPP_\alpha) \ge 1 - \alpha$.

%% file: equivalent-confidence-sets.tex
\section{Equivalence to confidence sets from testing}
\label{sec:equivalence}

As we discuss in the introduction, the original PPI
algorithm~\cite{AngelopoulosBaFaJoZr23} constructs confidence sets by testing the null that $\nabla L(\theta) = 0$ individually for each $\theta \in \R^d$.
A few special cases, including means, one-dimensional quantiles, and linear
regression coefficients, admit efficient implementations; however, the
general algorithm is intractable. Algorithm~\ref{alg:general_ppi}, on the
other hand, is computationally straightforward.
This computational efficiency comes at no cost in terms of statistical
power: we show that the confidence sets from Theorem~\ref{theorem:normality}
with $\lambda = 1$ are asymptotically equivalent to a careful choice of
confidence sets~\eqref{eq:CPP-original} obtained by testing $\nabla
L(\theta) = 0$.

To conduct the comparison, we study the asymptotic distribution of the test
statistics the two approaches rely on. Because the original
prediction-powered inference only considers the unweighted variant ($\lambda
= 1$), we focus on the estimate~\eqref{eq:ppi} and not~\eqref{eq:ppi-lambda}.
Theorem~\ref{theorem:normality} suggests that
given any estimate $\Sigmahat \cp \Sigma$, we consider the
test statistic
\begin{equation*}
  T_n(\theta) \defeq \sqrt{n} \Sigmahat^{-1/2} (\thetaPP - \theta).
\end{equation*}
For the original approach~\eqref{eq:CPP-original}, letting
$\VhatDelta(\theta) := \widehat\Var_n(\nabla \loss_{\theta} - \nabla
\loss_{\theta}^f)$ and $\Vhatf(\theta) := \widehat\Var_{N} (\nabla
\loss_{\theta}^f)$, the relevant test statistic is
\begin{equation*}
  U_n(\theta) \defeq \sqrt{n} \left(\VhatDelta(\theta) + \frac n N \Vhatf (\theta) \right)^{-1/2}
  \nabla \LPP(\theta).
\end{equation*}
Both approaches leverage pivotal asymptotic normality; in particular,
\begin{equation*}
  T_n(\theta\opt) \cd \normal(0, I_d)
  ~~ \mbox{and} ~~
  U_n(\theta\opt) \cd \normal(0, I_d).
\end{equation*}
Therefore, each implies an asymptotically exact $\chi^2$-based
$(1 - \alpha)$-confidence set
\begin{equation*}
\C_\alpha^{\mathrm{PP},T} \defeq \left\{\theta : \ltwo{T_n(\theta)}^2
  \le \chi^2_{d,1-\alpha}\right\}
  ~~ \mbox{and} ~~
\C_\alpha^{\mathrm{PP},U} \defeq \left\{\theta : \ltwo{U_n(\theta)}^2 \le \chi^2_{d,1-\alpha}
  \right\},
\end{equation*}
which satisfy $\lim_n \P(\theta\opt \in \C_\alpha^{\textup{PP},T}) = 1 -
\alpha$ and $\lim_n \P(\theta\opt \in \C_\alpha^{\textup{PP},U}) = 1 -
\alpha$. We remark in passing that $\C_\alpha^{\textup{PP},U}$ has
asymptotically correct level $\alpha$, while the original PPI
approach~\cite[Thm.~B.1]{AngelopoulosBaFaJoZr23} applies a union bound over
the $d$ coordinates and is thus conservative, making $\C_\alpha^{\textup{PP},U}$
tighter than the original PPI confidence set and the coming
equivalence stronger.

Theorem \ref{theorem:confidence-set-equivalence} shows the asymptotic
equivalence of $\C_\alpha^{\textup{PP},T}$ and
$\C_\alpha^{\textup{PP},U}$. For this, we require an additional condition
beyond the losses $\ell_\theta$ being smooth enough
(Definition~\ref{definition:smooth-enough}), which we term \emph{stochastic
smoothness}. To avoid distracting technical definitions, we state it in
Assumption~\ref{assumption:donsker} in
Appendix~\ref{sec:proof-confidence-set-equivalence}, which includes the
proof of Theorem~\ref{theorem:confidence-set-equivalence}. The condition
holds if, for example, $\nabla \ell_{\theta}$ is locally Lipschitz in a
neighborhood of $\theta\opt$, but can also hold for non-differentiable
losses.


\begin{theorem}
  \label{theorem:confidence-set-equivalence}
  Let the loss functions $\loss_\theta$ be smooth enough
  (Definition~\ref{definition:smooth-enough}) and
  in addition assume they are stochastically smooth
  (Assumption~\ref{assumption:donsker}).  Then $\C_\alpha^{\mathrm{PP},T}$
  and $\C_\alpha^{\mathrm{PP},U}$ are asymptotically equivalent in that, for
  any $c < \infty$,
  \begin{equation*}
    \limsup_n
    \sup_{\norm{h} \le c}
    \left|
    \P(\theta\opt + h/\sqrt{n} \in \C_\alpha^{\mathrm{PP},T})
    - \P(\theta\opt + h/\sqrt{n} \in \C_\alpha^{\mathrm{PP},U}) \right|
    = 0.
  \end{equation*}
  Furthermore, for any $t_n \gg \frac{1}{\sqrt{n}}$ and $c>0$, $\limsup_n
    \sup_{\norm{h} \ge c}
    \P(\theta\opt + h t_n \in \C_\alpha^{\mathrm{PP},T})
    = 0.$
\end{theorem}

In words, the probability of including any point within $\sim
\frac{1}{\sqrt{n}}$ of $\theta\opt$ in the confidence set is the same for \ppipp (with no power tuning) and a sharper variant of the
PPI proposal~\eqref{eq:CPP-original}
based on point-wise testing.
Moreover, this neighborhood of $\theta\opt$ is
the only relevant region of the parameter space: all points more than $\sim
\frac{1}{\sqrt{n}}$ away from the optimum have vanishing probability of
inclusion.
When we tune the parameter $\lambda$, as we consider in the next section,
the confidence sets \ppipp yields are only smaller, yielding
further improvements.

%% file: efficiency.tex

\section{Power tuning}
\label{sec:setting-lambda}

The last component of \ppipp is to adaptively tune the weighting parameter
$\lambda$ and thus, in addition to computational benefits, to achieve
statistical efficiency improvements over ``standard'' prediction-powered
inference and classical inference.
Towards this goal, we show how to optimally tune the weighting parameter
$\lambda$.
Along the way we make a few theoretical remarks
on the consequent statistical efficiency (see also the
note~\cite{AngelopoulosDuZr23w}).
While we cannot expect \ppipp to guarantee
full statistical efficiency for any distribution---this would require
modeling assumptions that
\ppipp explicitly avoids by providing a wrapper around black-box
prediction models $f$---we demonstrate
that it always improves over classical inference and in some cases is
fully statistically efficient.

\subsection{The optimal weighting parameter $\lambda$}

Our first step is to derive the value $\lambda\opt$ of $\lambda$ minimizing
the asymptotic variance $\Tr (\Sigma^\lambda)$ of the prediction-powered
estimate, where we recall the asymptotic covariance
\begin{equation*}
  \Sigma^\lambda
  = H_{\theta\opt}^{-1}
  \left(r \lambda^2 \cov(\nabla \loss_{\theta\opt}^f)
  + \cov(\nabla \loss_{\theta\opt} - \lambda \nabla \loss_{\theta\opt}^f)
  \right) H_{\theta\opt}^{-1}
\end{equation*}
from definition~\eqref{eqn:sigma-lambda}. One could straightforwardly choose
alternative scalarizations of $\Sigma^\lambda$, such as the maximum diagonal
entry $\max_j \Sigma_{jj}^\lambda$, to yield different minimizing
$\lambda$, but the principles remain the same.
(See Appendix~\ref{sec:proof-optimal-lambda} for the quick proof.)


\begin{proposition}
  \label{prop:optimal-lambda}
  The prediction-powered estimate $\thetaPP_\lambda$ has minimal asymptotic
  variance for
  \begin{equation*}
    \lambda\opt
    \defeq \argmin_\lambda \Tr(\Sigma^\lambda)
    = \frac{\Tr(H_{\theta\opt}^{-1}(
      \Cov(\nabla \loss_{\theta\opt}, \nabla \loss_{\theta\opt}^f)
      + \Cov(\nabla \loss_{\theta\opt}^f, \nabla \loss_{\theta\opt}))
      H_{\theta\opt}^{-1})}{
      2(1+r) \Tr(H_{\theta\opt}^{-1}\Cov(\nabla \loss_{\theta\opt}^f) H_{\theta\opt}^{-1} )}.
  \end{equation*}
\end{proposition}

A few remarks are in order here; we defer discussion of empirical implications
of this tuning to the next section and focus on analysis
here.
Substituting $\lambda\opt$ into the
expression~\eqref{eqn:sigma-lambda} for $\Sigma^\lambda$ yields optimal
asymptotic variance
\begin{equation*}
  \Tr(\Sigma^{\lambda\opt})
  = \Tr(H_{\theta\opt}^{-1}\Cov\left(\nabla \loss_{\theta\opt}\right)
  H_{\theta\opt}^{-1})
  - \frac{\Tr(H_{\theta\opt}^{-1}(\Cov(\nabla \loss_{\theta\opt},
    \nabla \loss_{\theta\opt}^f ) + \Cov(\nabla \loss_{\theta\opt}^f, \nabla \loss_{\theta\opt} ) )H_{\theta\opt}^{-1})^2}{4(1+r) \Tr(H_{\theta\opt}^{-1}\Cov(\nabla \loss_{\theta\opt}^f) H_{\theta\opt}^{-1} ) }.
\end{equation*}
The first term is the cumulative variance of the classical estimate, while
the second term is non-negative. Therefore, we see that with the optimal
choice of $\lambda$, \ppipp per
Algorithm~\ref{alg:general_ppi} (asymptotically) dominates classical
inference. Furthermore, the improvement over classical inference grows with
the correlation of the true gradients and those with predicted labels. By the Cauchy-Schwarz inequality and the law of total variance,
we have
\begin{align*}
  \Tr(\Sigma^{\lambda\opt})
  & \ge \Tr\left(H_{\theta\opt}^{-1} \Cov(\nabla \loss_{\theta\opt})
  H_{\theta\opt}^{-1}\right)
  - \frac{1}{1 + r}
  \Tr\left(H_{\theta\opt}^{-1} \Cov(\E[\nabla \loss_{\theta\opt} \mid X])
  H_{\theta\opt}^{-1}\right) \\
  & =
  \frac{r}{1 + r}
  \Tr\left(H_{\theta\opt}^{-1} \Cov(\nabla \loss_{\theta\opt}) H_{\theta\opt}^{-1}
  \right)
  + \frac{1}{1 + r} \Tr\left(H_{\theta\opt}^{-1}\E[\Cov(\nabla\loss_{\theta\opt} \mid X)]
  H_{\theta\opt}^{-1}\right),
\end{align*}
with equality in the first step if and only if
$\nabla \loss_{\theta\opt}(x, f(x)) = c\E[\nabla \loss_{\theta\opt}(x, Y)
  \mid X = x]$ for some fixed $c \ge 0$. The last bound is
the minimal asymptotic variance of \emph{any} estimator given
both a labeled and unlabeled sample (see~\cite{AngelopoulosDuZr23w}).
Essentially, \ppipp is always more efficient than classical estimators,
and it \emph{can} achieve optimal statistical efficiency if the black-box prediction function $f$ is sufficiently strong.

\subsection{A plug-in estimator and tuning procedure}

In finite samples, $\lambda\opt$ admits the natural
plug-in estimate
\begin{equation}
\label{eq:plugin_lambda}
\hat \lambda = \frac{\Tr(\widehat H_{\thetaPP}^{-1}( \widehat \Cov_n(\nabla \loss_{\thetaPP}, \nabla \loss_{\thetaPP}^f ) + \widehat \Cov_n(\nabla \loss_{\thetaPP}^f, \nabla \loss_{\thetaPP} ) ) H_{\thetaPP}^{-1} )}{2(1+\frac n N) \Tr( \widehat H_{\thetaPP}^{-1}\widehat\Cov_{N+n}(\nabla \loss_{\thetaPP}^f) \widehat H_{\thetaPP}^{-1} ) },
\end{equation}
where $\thetaPP = \thetaPP_\lambda$ is any estimate obtained by taking a
fixed $\lambda \in [0, 1]$, as each of these is $\sqrt{n}$-consistent
(Theorem~\ref{theorem:normality}), and $\what{H}_\theta = \nabla^2
L_n(\theta)$ is the empirical Hessian. The estimate $\hat\lambda$ is
consistent, satisfying
$\hat\lambda \cp \lambda\opt$, as (when the losses are smooth
enough as in Definition~\ref{definition:smooth-enough}) each individual term
consistently estimates its population counterpart.
Theorem~\ref{theorem:normality} thus implies the following corollary.
\begin{corollary}
  \label{cor:clt_w_lambdahat}
  Let $\hat\lambda$ be the estimate \eqref{eq:plugin_lambda} and
  the conditions of Theorem~\ref{theorem:normality} hold. Then
\[\sqrt{n}(\thetaPP_{\hat\lambda}- \theta\opt)\cd \normal(0,\Sigma^{\lambda\opt}).\]
\end{corollary}

The astute reader may notice that both $\lambda\opt$ and $\hat{\lambda}$ may
fail to lie on the region $[0, 1]$, which is the only region in which we can
guarantee convexity---and thus tractability of computing $\thetaPP_{\hat\lambda}$---even for GLMs. (This may occur, for example, if the
predictions $f(X_i)$ are anti-correlated with the true outcomes $Y_i$;
see Example~\ref{example:mean-tuning}.)  In
this case, we can take one of two approaches: the first, naively, is to
simply clip $\hat{\lambda}$ to lie in $[0, 1]$ in Algorithm~\ref{alg:glms}
or \ref{alg:general_ppi}. The second is to use classical one-step
estimators~\cite[e.g.][Chapter 5.7]{VanDerVaart98}. In this case, we
recognize that Theorem~\ref{theorem:normality} implies that for any fixed
$\lambda$, $\thetaPP_\lambda$ is $\sqrt{n}$-consistent, and take a single
Newton step on $\LPP_{\hat\lambda}$,
\begin{equation*}
  \thetaPP_{\textup{os}}
  \defeq \thetaPP_\lambda - \nabla^2 \LPP_{\hat{\lambda}}(
  \thetaPP_\lambda)^{-1} \nabla \LPP_{\hat\lambda}(\thetaPP_\lambda).
\end{equation*}
We call $\thetaPP_{\textup{os}}$ and the resulting inferences ``one-step
PPI.''
The difference between the one-step estimator and $\thetaPP_{\hat\lambda}$,
where we clip $\hat\lambda$ to be in $[0, 1]$, is usually empirically
immaterial; we explore this in Appendix~\ref{app:one-step-ppi}. We have the
following corollary~\cite[Thm.~5.45]{VanDerVaart98}.
\begin{corollary}
\label{cor:one-step}
  Let $\hat{\lambda}$ be the estimate~\eqref{eq:plugin_lambda}
  and the conditions of Theorem~\ref{theorem:normality} hold. Then
  \begin{equation*}
    \sqrt{n} (\thetaPP_{\textup{os}} - \theta\opt)
    \cd \normal(0, \Sigma^{\lambda\opt}).
  \end{equation*}
\end{corollary}
Using the plug-in estimate \eqref{eq:plugin_lambda} thus yields a
prediction-powered point estimate with minimal asymptotic variance, which in
turn implies smaller confidence intervals than both PPI according to  problem~\eqref{eq:ppi} and classical inference.


\begin{example}[Optimal tuning for mean estimation]
  \label{example:mean-tuning}
  The mean-estimation objective with \ppipp is always
  convex for any choice of $\lambda \in \R$, thus obviating the need
  for clipping or one-step corrections.
  Let $\theta\opt = \E[Y] = \argmin_\theta \E[\loss_\theta(Y)]$
  for the loss $\loss_\theta(y) = (y-\theta)^2$.
  Then for any $\lambda \in \R$, we have
  $\thetaPP_\lambda = \frac{1}{n} \sum_{i = 1}^n Y_i
  + \lambda(\frac{1}{N} \sum_{i = 1}^N f(\wt{X}_i) - \frac{1}{n}
  \sum_{i = 1}^n f(X_i))$.
  Substituting into the
  expression for $\lambda\opt$ from Proposition \ref{prop:optimal-lambda}
  gives
  \begin{equation*}
    \lambda\opt = \frac{\Cov(Y,f(X))}{(1+r)\Var(f(X))}.
  \end{equation*}
  We use the plug-in estimator
  \begin{equation*}
    \hat\lambda = \frac{\widehat\Cov_n(Y_i,f(X_i))}{(1+\frac n N )\widehat \Var_{N+n}(f(X_i))}
  \end{equation*}
   of $\lambda\opt$ in our experiments.
\end{example}



%% file: experiments.tex

\section{Experiments}
\label{sec:experiments}

We evaluate the benefits of the proposed methodology in simulation and on
real data. First, we study the effect of power tuning by comparing
power-tuned PPI (\ppipp) to both classical and prediction-powered inference
(PPI). We also compare to the original PPI approach for generic M-estimators
involving gridding $\R^d$ to test $\nabla \LPP(\theta) = 0$. Code implementing the \ppipp methods used in the experiments is available in the \texttt{ppi-py} package.\footnote{\texttt{https://github.com/aangelopoulos/ppi\_py}}

\subsection{The benefits of power tuning}

Our first set of experiments investigates tuning the weighting $\lambda$ via
Proposition~\ref{prop:optimal-lambda} and the
plug-in~\eqref{eq:plugin_lambda}, which we use in the
estimator~\eqref{eq:ppi-lambda} for $\thetaPP_{\hat{\lambda}}$. We first
perform experiments with simulated data, then revisit the
real-data experiments that \citet{AngelopoulosBaFaJoZr23} consider.
In all experiments, we use the confidence level $\alpha=0.1$.

\subsubsection{Simulation studies}
\label{sec:simulations}

Each simulation experiment follows a similar pattern.
We sample two i.i.d.\ datasets, $\{(X_i,Y_i)\}_{i=1}^n$ and $\{(\tX_i,
\tY_i)\}_{i=1}^N$, the first serving as the labeled data and the second as
the unlabeled data (so the algorithms have no access to
$\{\tY_i\}_{i=1}^N$).
We corrupt the labels to form simulated predictions $\{f(X_i)\}_{i=1}^n$ and
$\{f(\tX_i)\}_{i=1}^N$ by adding noise to the true labels $Y_i$ and
$\tY_i$.
We compute the average coverage and confidence interval width over 100 trials for the following methods:
\begin{itemize}
\item PPI with no power tuning, equivalent to setting $\lambda=1$ (green);
\item Classical inference, equivalent to setting $\lambda=0$ (blue);
\item Power-tuned PPI using the estimate~\eqref{eq:plugin_lambda}
  for $\hat{\lambda}$ (orange).
\end{itemize}
We set $N=10000$ and vary the labeled sample size $n$
as well as the noise in the predictions.
We choose the noise levels heuristically to represent low, medium, and high noise.

\paragraph{Mean estimation.}

The goal is to estimate the mean outcome $\E[Y]$, where $Y \sim \normal(0,1)$.
We form predictions as $f(X_i) = Y_i + \sigma \epsilon_i$, $\epsilon_i \simiid \normal(0, 1)$.
We set the noise level $\sigma$ to be $0.1$, $1$, and $2$, successively.
There are no covariates in this problem.
We show the results in Figure~\ref{fig:mean_estimation}. We see that tuned PPI is essentially never worse than either baseline.
When the noise is low, \ppipp behaves similarly to standard PPI
($\lambda=1$). When the noise is high, \ppipp behaves like classical
inference ($\lambda=0$). In the intermediate noise regime, \ppipp
significantly outperforms both baselines.

\begin{figure}[t]
    \centering
    \includegraphics[width=0.9\textwidth]{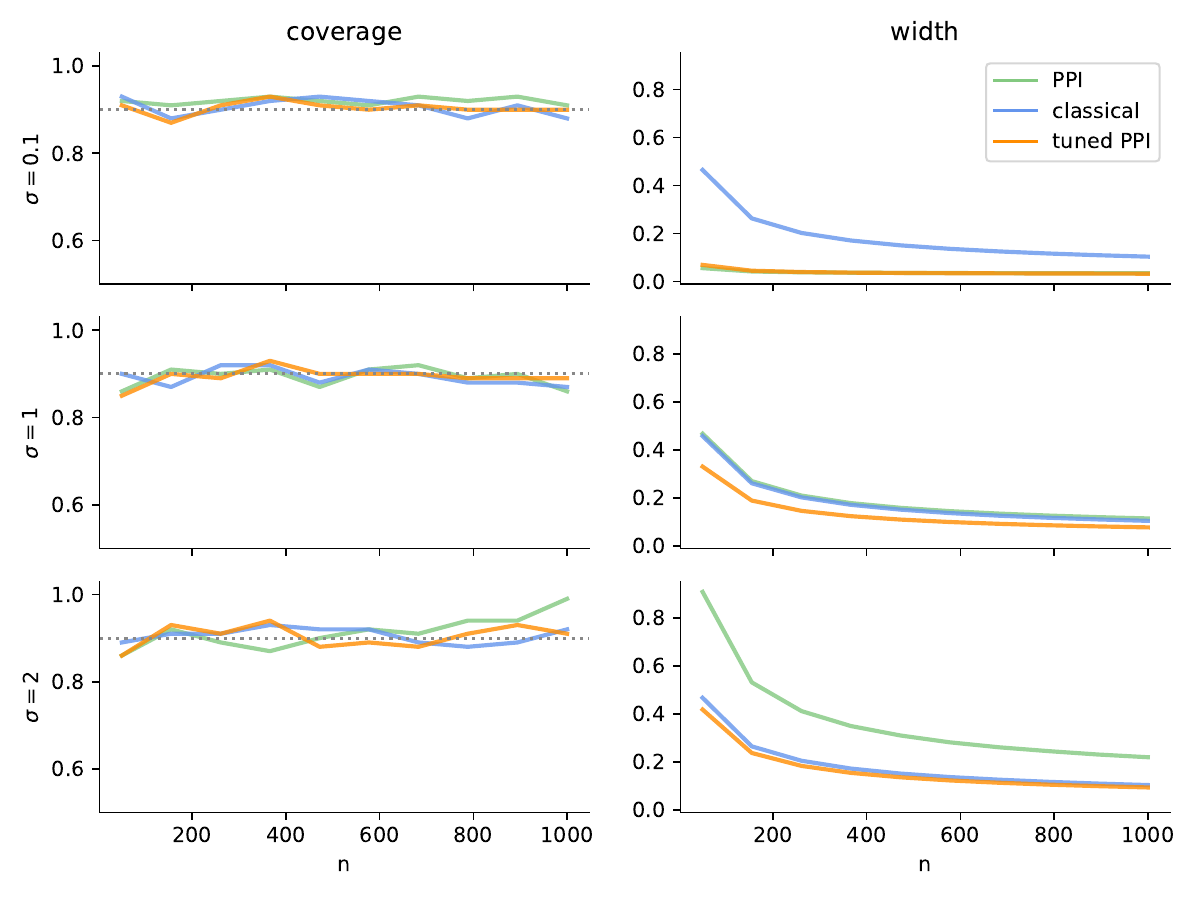}
    \caption{Mean estimation simulation study.  The left column shows
      coverage, the right column shows confidence interval width.  The top,
      middle, and bottom rows correspond to noise levels $\sigma=0.1$,
      $\sigma=1$, and $\sigma=2$, respectively.}
    \label{fig:mean_estimation}
\end{figure}

\paragraph{Linear regression.}

The goal is to estimate the coefficients of a linear regression model $Y =
X^\top \theta + \epsilon$, where $X \sim \normal(0,I_d)$, $\theta \in \R^d$,
and $\epsilon \sim \normal(0,1)$, in dimension $d = 2$.
We use predictions $f(X_i) = X_i^\top \theta + \epsilon_i$,
where $\epsilon_i \sim \normal(-2, \sigma^2)$.
We vary the noise level $\sigma \in \{0.1, 0.5, 1\}$, plotting
results in Figure~\ref{fig:linear_regression}.
The results are qualitatively similar to those in
Figure~\ref{fig:mean_estimation}: \ppipp adaptively interpolates between
classical inference and standard PPI as the noise level changes,
outperforming both.

\begin{figure}[t]
  \centering
  \includegraphics[width=0.9\textwidth]{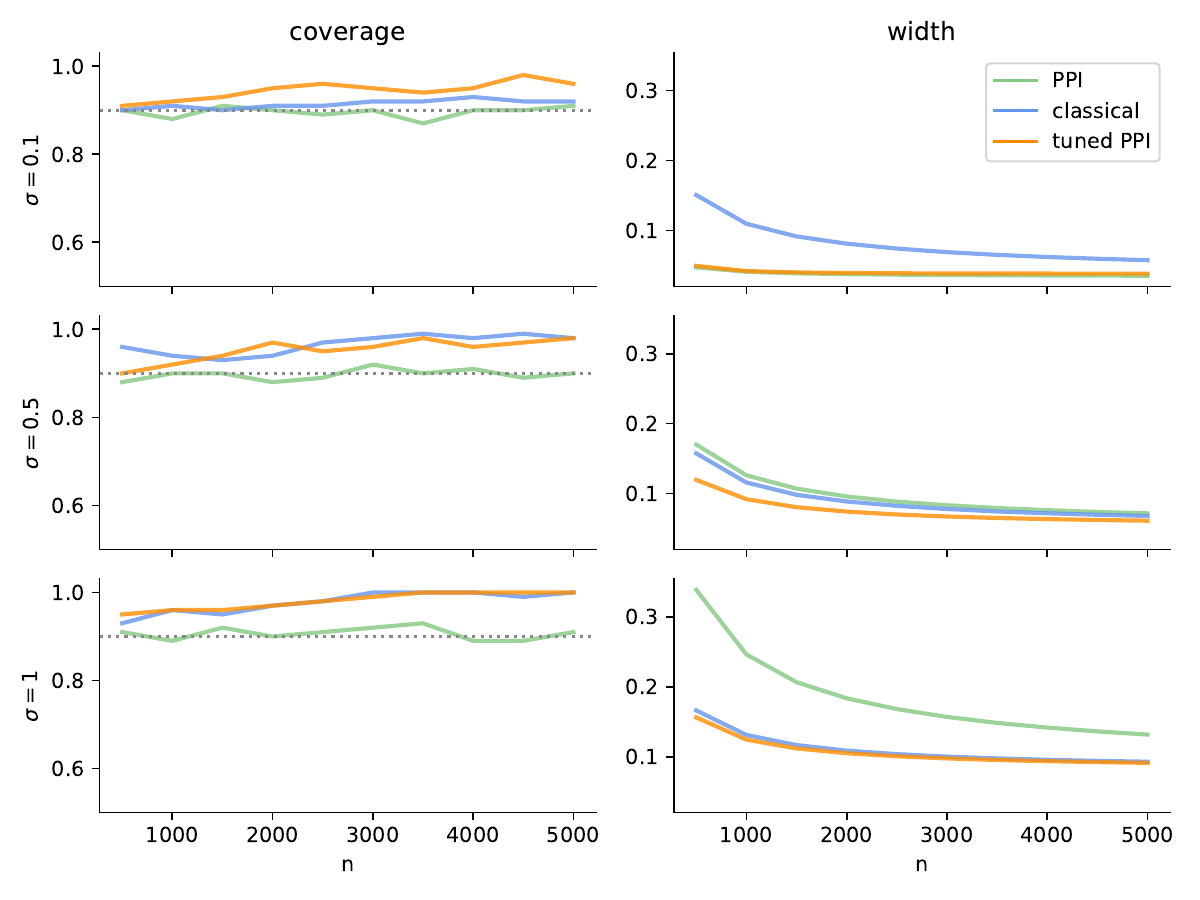}
  \caption{Linear regression simulation study.  The left column shows
    coverage, the right column shows confidence interval width. The top,
    middle, and bottom rows correspond to noise levels $\sigma=0.1$,
    $\sigma=0.5$, and $\sigma=1$, respectively.}
  \label{fig:linear_regression}
\end{figure}

\paragraph{Logistic regression.}

The goal is to estimate the coefficients of a logistic regression model,
$\P(Y=1) = \mu_\theta(X)$, where $\mu_\theta(x) = \frac{e^{\theta^\top x}}{1
  + e^{\theta^\top x}}$ for $\theta \in \R^d$ with dimension $d = 2$.
We draw $X \sim \normal(0,I_d)$ and to simulate a binary classifier $f$ with
error $\sigma$, we set the predictions $f(X_i)$ to be $Y_i$ but randomly
flipped with probability $\sigma$.
We vary the flipping probability $\sigma \in \{0.01, 0.1, 0.2\}$, presenting
results in Figure~\ref{fig:logistic_regression}.
As before, \ppipp successfully adapts to the varying noise in the
predictions.

\begin{figure}[t]
    \centering
    \includegraphics[width=0.9\textwidth]{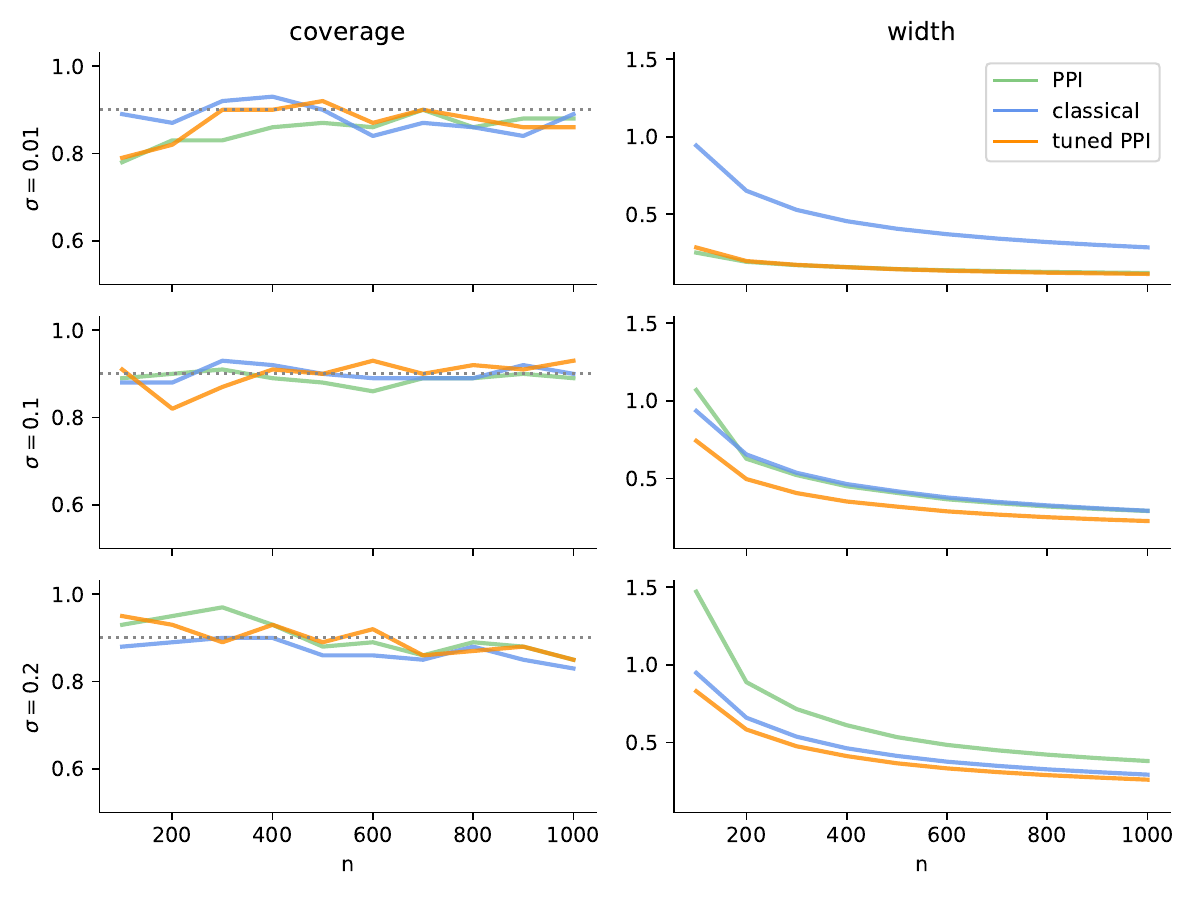}
    \caption{Logistic regression simulation study. The left column shows
      coverage, the right column shows confidence interval width. The top,
      middle, and bottom rows correspond to noise levels $\sigma=0.01$,
      $\sigma=0.1$, and $\sigma=0.2$, respectively.}
    \label{fig:logistic_regression}
\end{figure}

\subsubsection{Real data experiments}

We now revisit five experiments on real data that
\citeauthor{AngelopoulosBaFaJoZr23} perform~\citep{AngelopoulosBaFaJoZr23},
where we replicate their experimental settings, comparing to
classical inferential approaches (that do not use the unlabeled data)
as well.

\paragraph{Mean estimation on deforestation data.}

The goal is to estimate the mean deforestation rate in the Amazon based on
satellite imagery data~\cite{AngelopoulosBaFaJoZr23}.  The dataset includes
$1596$ total pairs of observations and predictions, where the observations
indicate whether deforestation occurred in a given randomly sampled region
and the predictions are probits corresponding to these indicators.  We take
these predictions from the same gradient-boosted classifier as
in~\cite{AngelopoulosBaFaJoZr23}, which is fit on a publicly available
machine learning system that released tree canopy forest cover predictions
at 30m resolution in the years 2000 and 2015.  Generally, when the amount of
forest cover in 2015 is substantially lower than in 2000, the model predicts
that deforestation occurred.

Of the $1596$ total data points, we vary $n$, and then
randomly take a labeled subset of size $n$ and
use $N=1596-n$ as unlabeled.
Figure~\ref{fig:mean_estimation_real_data} shows the results of the
procedures over $100$ random splits into labeled and unlabeled data. For small
$n$, \ppipp behaves similarly to standard PPI. When $n$ is large, meaning
the unlabeled dataset size $N$ is small, classical inference outperforms
standard PPI; \ppipp outperforms both.

\begin{figure}[H]
    \centering
    \includegraphics[width=0.9\textwidth]{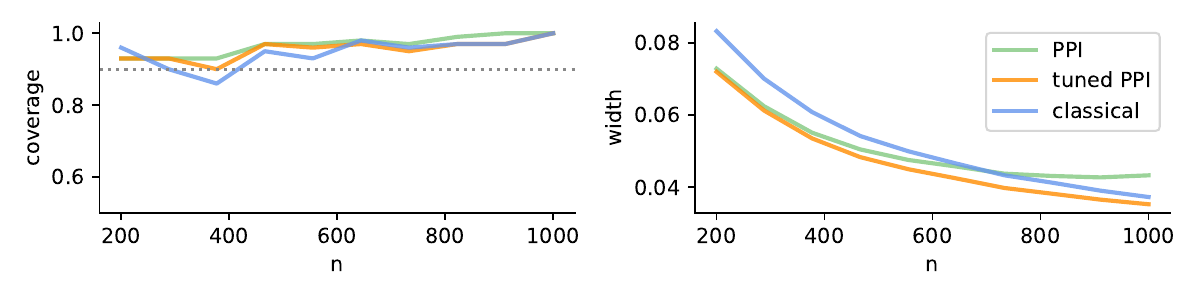}
    \caption{Mean estimation on deforestation data. The left panel shows coverage and the right shows width.}
    \label{fig:mean_estimation_real_data}
\end{figure}

\paragraph{Mean estimation on galaxy data.}

The goal is to estimate the fraction of spiral galaxies in the local universe based on images from the Sloan Digital Sky Survey (SDSS) and a pre-trained ResNet classifier, as in~\cite{AngelopoulosBaFaJoZr23}.
The scale of the universe makes human labeling of galaxies impossible, so
scientists wish to use a computer-vision-based classifier to study the
development of the universe. The classifier ingests images of galaxies and
outputs an estimated probability that the galaxy is spiral.
The dataset includes 1,364,122 total observations, of which we randomly take $n$ to serve as the labeled data and $N=1,\!364,\!122-n$ as the unlabeled data, for varying $n$.
We show the results in Figure~\ref{fig:mean_estimation_real_data_galaxies} over $100$ random splits into labeled and unlabeled data.
As the predictions are very accurate in this problem, \ppipp and standard PPI are basically indistinguishable, and both are significantly more powerful than classical inference.

\begin{figure}[H]
    \centering
    \includegraphics[width=0.9\textwidth]{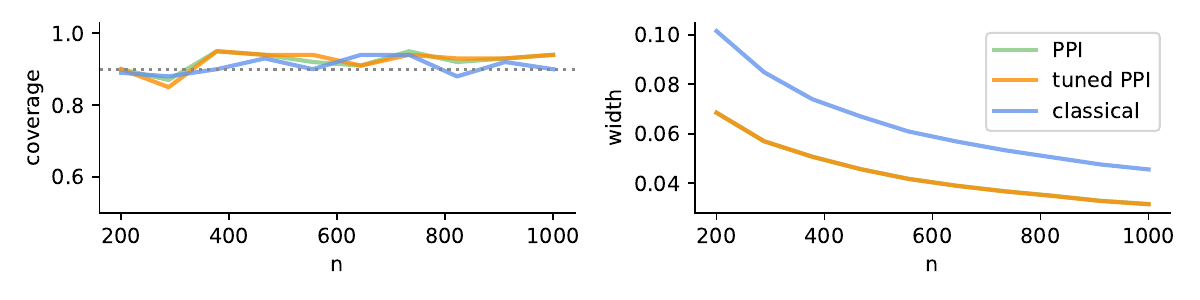}
    \caption{Mean estimation on galaxy data. The left panel shows coverage and the right shows width.}
    \label{fig:mean_estimation_real_data_galaxies}
\end{figure}

\paragraph{Odds ratio estimation on AlphaFold data.}

The goal is to estimate the odds ratio of a protein being phosphorylated and
being part of an intrinsically disordered region (IDR) using the predictions
from AlphaFold~\cite{JumperEvEtAl21}, as in~\cite{AngelopoulosBaFaJoZr23}.
The cost of measuring protein structure has prompted scientists to
increasingly rely on AlphaFold's predictions to study the
natural distribution of protein structures. The data set includes $10802$
total observations of ground-truth IDR indicators and corresponding
AlphaFold predictions, which we randomly split into $n$ labeled
and $N=10802-n$ unlabeled data points, for varying $n$.  See
Figure~\ref{fig:odds_ratio_estimation_real_data} for the results over $100$
random splits into labeled and unlabeled data. For moderately low values of
$n$, tuned PPI and standard PPI have similar performance, both outperforming
classical inference. As $n$ becomes large, tuned PPI gradually becomes more
powerful than both alternatives.

The original PPI confidence sets~\citep{AngelopoulosBaFaJoZr23} are formed by splitting the error budget between two means and post-processing them to recover a confidence set for the odds ratio. 
We do not use this overly conservative approach; instead, we derive an asymptotically exact confidence set for the odds ratio using the delta method.
See Appendix~\ref{app:odds-ratio} for details.

\begin{figure}[H]
    \centering
    \includegraphics[width=0.9\textwidth]{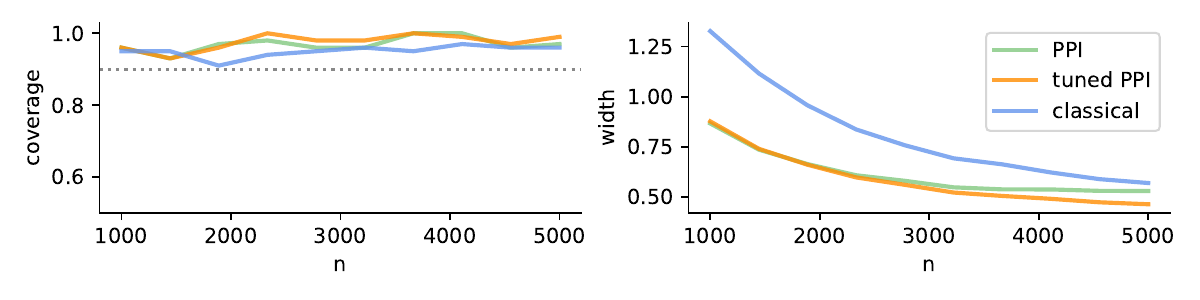}
    \caption{Odds ratio estimation on AlphaFold data. The left panel shows coverage and the right shows width.}
    \label{fig:odds_ratio_estimation_real_data}
\end{figure}

\paragraph{Linear regression on census data.}

In this experiment, the goal is to estimate the coefficients of a linear
regression model relating age (1-99) and sex (M/F) to income based on census
data from California, as in~\cite{AngelopoulosBaFaJoZr23}. We download the
data through the Folktables~\cite{DingHaMiSc21} interface and use an XGBoost
model~\cite{ChenGu16} taking as input the ten covariates available in
Folktables (including age, sex, the binary indicator of private health
insurance) to produce predictions.  We train the model on the entire census
dataset in the year 2018 and use it to predict income for the year
2019.  The 2019 data includes $377,\!575$ observations, of which we randomly
take $n$ as labeled and $N=377,\!575-n$ as unlabeled, for varying $n$.  We show
the results in Figure~\ref{fig:ols_real_data} over $100$ random splits into
labeled and unlabeled data. \ppipp is visually indistinguishable from
standard PPI with $\lambda=1$.

\begin{figure}[ht]
    \centering
    \includegraphics[width=0.9\textwidth]{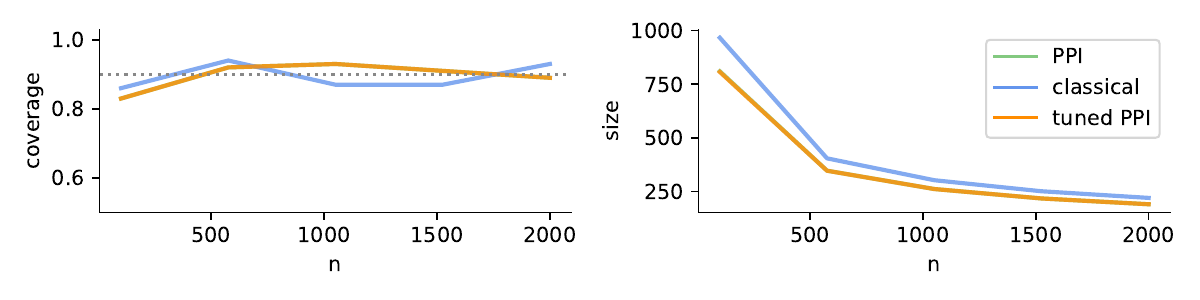}
    \caption{Linear regression on census data. The left panel shows coverage and the right shows width.}
    \label{fig:ols_real_data}
\end{figure}

\paragraph{Logistic regression on census data.}
The goal is to estimate the coefficients of a logistic regression model relating income (\$) to the binary indicator of private health insurance based on census data.
The setup otherwise resembles the previous experiment: we train an XGBoost model to predict income from ten other covariates, we have $377,\!575$ total observations, and so on.
Figure~\ref{fig:logistic_real_data} shows the results over $100$ random splits into labeled and unlabeled data. Tuned PPI behaves similarly to PPI, and both are more powerful than classical inference for all values of $n$.

\begin{figure}[ht]
    \centering
    \includegraphics[width=0.9\textwidth]{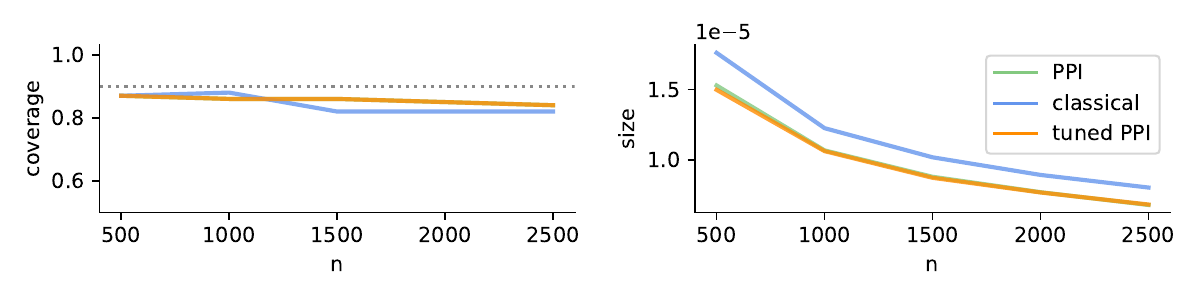}
    \caption{Logistic regression on census data. The left panel shows coverage and the right shows width.}
    \label{fig:logistic_real_data}
\end{figure}

\subsection{Comparison to confidence sets from testing}

The final comparison is to the naive implementation of prediction-powered
inference~\citep[Algorithm 3]{AngelopoulosBaFaJoZr23}.  We do not apply
power tuning in this comparison as the goal is to understand the independent
effect of the gridding procedure, which tests for ``each'' $\theta$ whether
$\nabla \LPP(\theta) = 0$ is plausible, versus the technique proposed
herein. The gridding means the original PPI method has much larger
computational cost than the proposed method, while it is also less accurate
(leading to larger sets, as in Figure~\ref{fig:comparison}).  This follows
because, to remain valid, the gridding method necessarily includes
an extra grid point along each dimension.  Furthermore, the method often
outputs infinite sets because, in an effort to avoid an infinite runtime,
the algorithm simply terminates grid refinement after a certain number of
iterations.  See Figure~\ref{fig:comparison} for the median interval width
for varying dimension $d$ and limit on the number of grid points before the
algorithm returns an infinite set.

\begin{figure}[ht]
    \centering
    \includegraphics[width=0.9\textwidth]{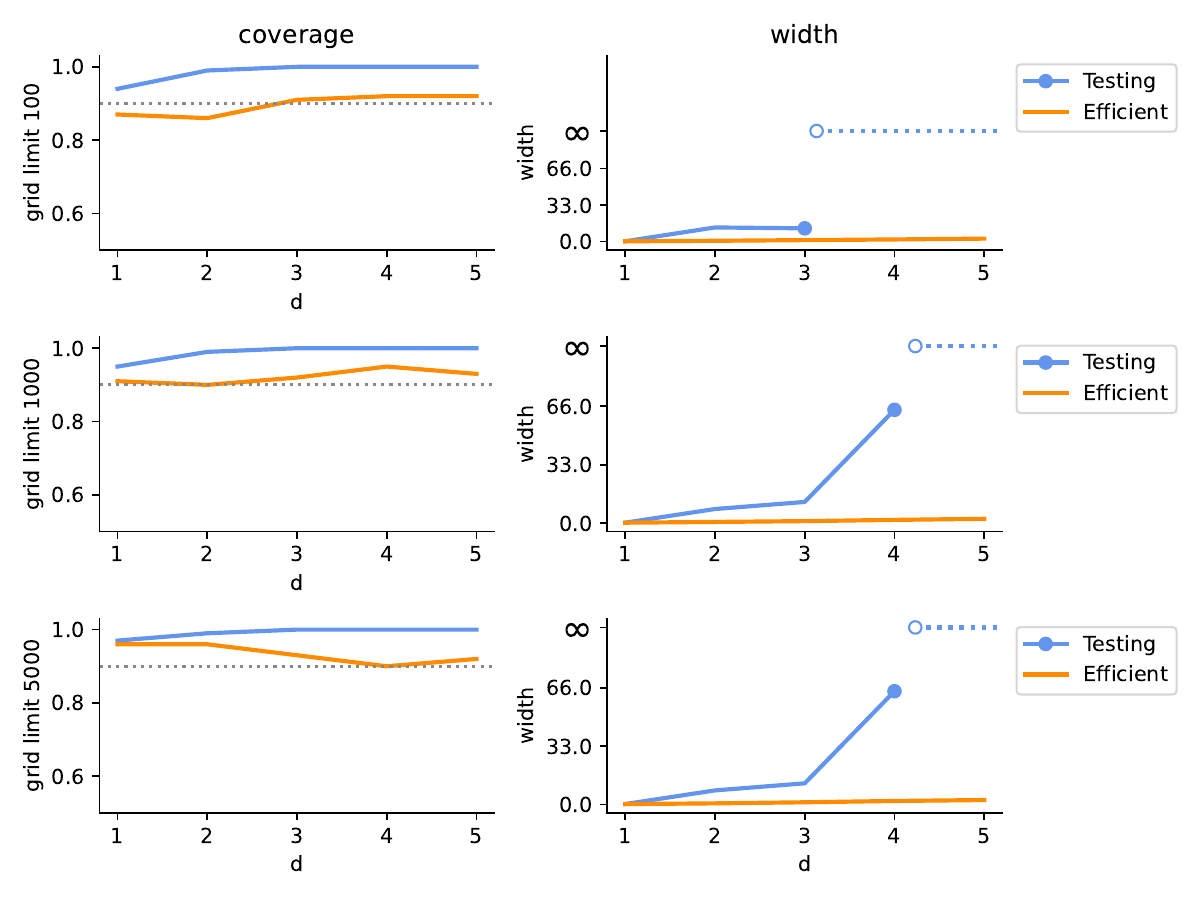}
    \caption{Comparison of \ppipp logistic regression, as in
      Algorithm~\ref{alg:glms}, and standard PPI logistic
      regression~\cite[Algorithm 3]{AngelopoulosBaFaJoZr23}. The left panel
      shows coverage and the right shows width.}
    \label{fig:comparison}
\end{figure}

\section*{Acknowledgements}
The authors would like to acknowledge the support of the Mohamed bin Zayed
University of Artificial Intelligence and president Eric Xing as well as
Barbara Engelhardt for their organization of the Michael Jordan Symposium,
which led to this idea and collaboration.  We also thank Michael Jordan for
his encouragement.  JCD's work was additionally supported by the Office of
Naval Research Grant N00014-22-1-2669 and as a Thomas and Polly Bredt
Faculty Scholar.

%% file: proofs.tex

\section{Proofs}

\subsection{Proof of Theorem \ref{theorem:normality}}
\label{sec:proof-normality}

We formally state the smoothness condition needed for Theorem \ref{theorem:normality}.

\begin{definition}[Smooth enough losses]
  \label{definition:smooth-enough}
  The loss $\loss_\theta$ is \emph{smooth enough} if
  \begin{enumerate}[label=(\roman*),leftmargin=1.5em]
  \item  
  \label{item:loss-differentiability} the losses
    $\loss_\theta(x, y)$ and $\loss_\theta(x, f(x))$ are
    differentiable at $\theta\opt$ for $\P$-almost
    every $(x, y)$,
  \item \label{item:locally-lipschitz}
    the losses $\loss_\theta$ are
    locally Lipschitz around $\theta\opt$: there is a neighborhood
    of $\theta\opt$ such that
    $\loss_\theta(x, y)$ is $\lipobj(x, y)$-Lipschitz
    and $\loss_{\theta}(x, f(x))$ is $\lipobj(x)$-Lipschitz in $\theta$,
    where $\E[\lipobj(X, Y)^2 + \lipobj(X)^2] < \infty$,
  \item the population losses
    $L(\theta) = \E[\loss_\theta(X, Y)]$ and
    $L^f(\theta) = \E[\loss_\theta(X, f(X))]$ have Hessians $H_{\theta\opt} = \nabla^2 L(\theta\opt) \succ 0$ and
    $H_{\theta\opt}^f = \nabla^2 L^f(\theta\opt)$.
  \end{enumerate}
\end{definition}

The proof follows that of \citet[Theorem
  5.23]{VanDerVaart98}, with the modifications necessary
to address the different observation model and empirical objective.  Throughout the proof, we recall the labeling
conventions that a tilde $\wt{\cdot}$ indicates the unlabeled sample
$\{\wt{X}_i\}_{i = 1}^N$ and a superscripted $f$ (i.e.\ $\cdot^f$) indicates
using the machine-learning predictions $f$.

With this notational setting,
given a function $g$, we use the shorthand notation
\begin{align}
  \E_n g &:= \frac{1}{n} \sum_{i=1}^n g(X_i,Y_i), \quad \GG_n g = \sqrt{n}\left(\E_n g - \E[g(X,Y)]\right); \label{eq:Gn}\\
  \widetilde \E_N^f g &:= \frac{1}{N} \sum_{i=1}^N g(\widetilde X_i, f(\widetilde X_i)),\quad \widetilde \GG_N^f g := \sqrt{N} ( \widetilde \E_N^f g - \E[g(X,f(X))]); \label{eq:GNf}\\
  \E_n^f g &:= \frac{1}{n} \sum_{i=1}^n g( X_i, f(X_i)), \quad \GG_n^f g := \sqrt{n}  (\E_n^f g - \E[g(X,f(X))]) \label{eq:Gnf}.
\end{align}

The differentiability of $\loss_\theta(x,y)$ and
$\loss_{\theta}(x, f(x))$ at $\theta\opt$ (Definition~\ref{definition:smooth-enough}\ref{item:loss-differentiability}) and the local Lipschitzness of the losses
(Definition~\ref{definition:smooth-enough}\ref{item:locally-lipschitz})
imply that
for every (possibly random) sequence $h_n = O_P(1)$, we have
\begin{equation*}
  \GG\left[\sqrt{n} \left(\loss_{\theta\opt + \frac{h_n}{\sqrt{n}}} - \loss_{\theta\opt} \right) - h_n^\top \nabla \loss_{\theta\opt} \right]
  \cp 0
\end{equation*}
for any of the empirical processes $\GG\in \{\GG_n, \widetilde \GG_N^f,
\GG_n^f \}$ (see \cite[Lemma 19.31]{VanDerVaart98}).  By applying a second-order Taylor expansion, this implies
\begin{align*}
  n \E_n\left(\loss_{\theta\opt + \frac{h_n}{\sqrt{n}}} - \loss_{\theta\opt}\right)
  & =
  n \left(L\Big(\theta\opt + \frac{h_n}{\sqrt{n}}\Big)
  - L(\theta\opt)\right)
  + h_n^\top \GG_n \nabla \loss_{\theta\opt} + o_P(1) \\
  & = \half h_n^\top H_{\theta\opt} h_n
  + h_n^\top \GG_n \nabla\loss_{\theta\opt} + o_P(1).
\end{align*}
Similarly, we have the two equalities
\begin{align*}
  \hat\lambda n \widetilde \E_N^f\left(\loss_{\theta\opt + \frac{h_n}{\sqrt{n}}} - \loss_{\theta\opt}\right) &= (\lambda + o_P(1))\left(\frac{1}{2} h_n^\top H_{\theta\opt}^f h_n + \sqrt{\frac{n}{N}} h_n ^\top \widetilde \GG_N^f \nabla \loss_{\theta\opt} + o_P(1)\right)\\
  & = \lambda \left(\frac{1}{2} h_n^\top H_{\theta\opt}^f h_n + \sqrt{\frac{n}{N}} h_n ^\top \widetilde \GG_N^f \nabla \loss_{\theta\opt}\right) + o_P(1),
  ~~ \mbox{and} \\
 - \hat\lambda n \E_n^f\left(\loss_{\theta\opt + \frac{h_n}{\sqrt{n}}} - \loss_{\theta\opt}\right) &= (\lambda + o_P(1))\left( -\frac{1}{2} h_n^\top H_{\theta\opt}^f h_n - h_n ^\top \GG_n^f  \nabla \loss_{\theta\opt} + o_P(1) \right)\\
 &= \lambda \left( -\frac{1}{2} h_n^\top H_{\theta\opt}^f h_n - h_n ^\top \GG_n^f  \nabla \loss_{\theta\opt}\right) + o_P(1)  \label{eq:G_nf}.
\end{align*}

Adding the preceding three equations and recalling the definition
that 
\[\LPP_\lambda(\theta) = \E_n \loss_\theta
+ \lambda (\wt{\E}_N^f \loss_\theta - \E_n^f \loss_\theta),\]
we have
\begin{equation*}
  n \left(\LPP_{\hat\lambda}\left(\theta\opt + \frac{h_n}{\sqrt{n}}\right)
  - \LPP_{\hat\lambda}(\theta\opt)\right)
  = \half h_n^\top H_{\theta\opt} h_n
  + h_n^\top \left(\GG_n \nabla\loss_{\theta\opt}
  + \lambda\sqrt{\frac{n}{N}} \widetilde \GG_N^f \nabla\loss_{\theta\opt}
  - \lambda \GG_n^f \nabla\loss_{\theta\opt}\right) + o_P(1).
\end{equation*}

We now evaluate this expression for two particular choices of the sequence
$h_n$. First, consider $h_n\opt = \sqrt{n}(\thetaPP_{\hat\lambda} -
\theta\opt)$.  Because the losses are smooth enough
(Definition~\ref{definition:smooth-enough}), $\theta \mapsto \lambda
\loss_\theta$ is locally Lipschitz uniformly for $\lambda$ in (any) compact set,
and so the consistency $\thetaPP_{\hat\lambda}\to \theta\opt$ implies that
$h_n\opt = O_P(1)$ (\citet[Corollary 5.53]{VanDerVaart98}).
This yields
\begin{equation*}
  n \left(\LPP_{\hat\lambda}(\thetaPP_{\hat\lambda}) - \LPP_{\hat\lambda}(\theta\opt)\right) = \frac{1}{2} {h_n\opt}^\top H_{\theta\opt} h_n\opt + {h_n\opt} ^\top \left(\GG_n \nabla\loss_{\theta\opt} + \lambda\sqrt{\frac{n}{N}} \widetilde \GG_N^f \nabla\loss_{\theta\opt} - \lambda \GG_n^f \nabla\loss_{\theta\opt}\right) + o_P(1).
\end{equation*}
We can perform a similar expansion with
$h_n = -H_{\theta\opt}^{-1}
(\GG_n \nabla\loss_{\theta\opt} + \lambda \sqrt{\frac{n}{N}} \widetilde
\GG_N^f \nabla\loss_{\theta\opt} - \lambda \GG_n^f
\nabla\loss_{\theta\opt})$, which is certainly $O_P(1)$,
yielding
\begin{equation*}
  n \left(\LPP_{\hat\lambda}(\theta\opt -h_n/\sqrt{n}) - \LPP_{\hat\lambda}(\theta\opt)\right)
  = -\half h_n^\top H_{\theta\opt} h_n + o_P(1).
\end{equation*}
By the definition of $\thetaPP_{\hat\lambda}$, the left-hand side of the
first equation is smaller than the left-hand side of the second equation,
and so
\begin{align*}
  \half {h_n\opt}^\top H_{\theta\opt} h_n\opt
  - {h_n\opt}^\top H_{\theta\opt} h_n
  \le - \half h_n^\top H_{\theta\opt} h_n + o_P(1),
\end{align*}
or, rearranging,
\begin{equation*}
  \half (h_n\opt - h_n)^\top H_{\theta\opt} (h_n\opt - h_n)
  = o_P(1).
\end{equation*}
As $H_{\theta\opt}$ is positive definite, we thus must have
$h_n\opt = h_n + o_P(1)$, that is,
\begin{equation*}
  \sqrt{n} \left(\thetaPP_{\hat\lambda} - \theta\opt\right)
  = -H_{\theta\opt}^{-1} \left(\GG_n \nabla \loss_{\theta\opt}
  + \lambda \sqrt{\frac{n}{N}} \wt{\GG}_N^f \nabla \loss_{\theta\opt}
  - \lambda \GG_n^f \nabla \loss_{\theta\opt}\right) + o_P(1).
\end{equation*}

It remains to demonstrate the asymptotic normality of the empirical
processes on the right-hand side. This follows by the central limit theorem:
\begin{align*}
&\GG_n \nabla\loss_{\theta\opt} + \lambda \sqrt{\frac{n}{N}} \widetilde \GG_N^f \nabla\loss_{\theta\opt} - \lambda \GG_n^f \nabla\loss_{\theta\opt}\\
&\quad = \lambda \sqrt{\frac{n}{N}} \cdot \frac{1}{\sqrt{N}} \sum_{i=1}^N (\nabla \tilde \loss_{\theta\opt,i}^f - \E[ \nabla \loss_{\theta\opt}^f])  + \frac{1}{\sqrt n} \sum_{i=1}^n \left( \nabla \loss_{\theta\opt,i} - \lambda \nabla \loss_{\theta\opt,i}^f - \E[\nabla \loss_{\theta\opt} -  \lambda\nabla \loss_{\theta\opt}^f]\right)\\
&\quad \cd \normal \left(0, \lambda^2 r \cdot \cov(\nabla \loss_{\theta\opt}^f) + \cov(\nabla \loss_{\theta\opt} - \lambda\nabla \loss_{\theta\opt}^f)\right).
\end{align*}
Thus $H_{\theta\opt}^{-1} \left(\GG_n \nabla\loss_{\theta\opt} + \lambda
\sqrt{\frac{n}{N}} \widetilde \GG_N^f \nabla\loss_{\theta\opt} - \lambda
\GG_n^f \nabla\loss_{\theta\opt}\right)\cd \normal(0,\Sigma^\lambda)$,
completing the proof.

\subsection{Proof of Proposition \ref{prop:consistency-convex}}
\label{sec:proof-consistency-convex}

We begin by stating an auxiliary lemma.

\begin{lemma}
  \label{lemma:convexity_growth}
  Let $g:\R^d\to\R$ be a convex function. Fix $\theta_0,\theta_1\in\R^d$ and
  $t\geq 1$, and define $\theta_t = \theta_0 + t(\theta_1 - \theta_0)$. Then
  $g(\theta_t) \geq g(\theta_0) + t(g(\theta_1) - g(\theta_0))$.
\end{lemma}
\begin{proof}
  Rearranging terms, we have $\theta_1 = \frac 1 t \theta_t + (1 - \frac 1
  t)\theta_0$ . Applying the definition of convexity at $\theta_1$ gives
  \[g(\theta_1) \leq \frac 1 t g(\theta_t) + \left(1-\frac 1 t \right)g(\theta_0),\]
  which is equivalent to $g(\theta_t) \geq g(\theta_0) + t(g(\theta_1) - g(\theta_0))$.
\end{proof}

Returning to Proposition~\ref{prop:consistency-convex}, the local
Lipschitzness condition implies there is an $\epsilon > 0$ such that
\begin{align}
  \label{eq:convergence_epsball}
  \sup_{\theta:\|\theta-\theta\opt\|\leq \epsilon} |\LPP_{\lambda}(\theta) - L(\theta)| \cp 0
\end{align}
by a standard covering argument.
By the uniqueness of $\theta\opt$, for all
$\epsilon>0$ we know there exists a $\delta>0$ such that $L(\theta) -
L(\theta\opt)\geq \delta$ for all $\theta$ on the $\epsilon$-shell $\{\theta
\mid \norm{\theta -
\theta\opt}=\epsilon\}$ around $\theta\opt$.
With this we can write
\begin{align*}
  \inf_{\norm{\theta-\theta\opt} = \epsilon}
  \LPP_{\lambda}(\theta) - \LPP_{\lambda}(\theta\opt)
  & = \inf_{\norm{\theta - \theta\opt} =\epsilon} \left((\LPP_{\lambda}(\theta) - L(\theta) + (L(\theta) - L(\theta\opt)) + (L(\theta\opt) - \LPP_{\lambda}(\theta\opt))\right)\\
  &\geq \delta - o_P(1),
\end{align*}
where the convergence~\eqref{eq:convergence_epsball}
guarantees that the first and third terms above vanish.

Now fix any $\theta$ such that $\norm{\theta-\theta\opt} \geq \epsilon$. We
apply Lemma~\ref{lemma:convexity_growth} with $\theta_0 = \theta\opt$ and
$\theta_1 = \frac{\theta-\theta\opt}{\|\theta-\theta\opt\|}\epsilon$.
Combining the preceding
display with Lemma~\ref{lemma:convexity_growth} gives
\begin{equation*}
  \LPP_{\lambda}(\theta) - \LPP_{\lambda}(\theta\opt) \geq \frac{\|\theta-\theta\opt\|}{\epsilon}(\delta-o_P(1)) \geq \delta - o_P(1).
\end{equation*}
Therefore, no $\theta$ with $\|\theta-\theta\opt\|\geq \epsilon$ can
minimize $\LPP_{\lambda}$, and so by contradiction we have shown
$\P(\norms{\thetaPP_{\lambda} - \theta\opt}< \epsilon)\to 1$, as desired.

We turn to the final claim of the proposition, which allows random
$\hat{\lambda}$. The convergence~\eqref{eq:convergence_epsball} holds
uniformly for $\lambda'$ in a small neighborhood of $\lambda$ as well
because of an identical covering argument; as $\hat{\lambda} = \lambda +
o_P(1)$, with probability tending to 1 we have that $\hat{\lambda}$ belongs to
the neighborhood of $\lambda$. The rest of the proof is identical.

\subsection{Proof of Theorem \ref{theorem:confidence-set-equivalence}}
\label{sec:proof-confidence-set-equivalence}

We will require that the loss function is stochastically smooth in the following sense.

\begin{assumption}[Stochastic smoothness]
  \label{assumption:donsker}
  There exists a compact neighborhood $B$ of $\theta\opt$ such that
  the classes
  \begin{equation*}
    \mc{L} \defeq \{(x, y) \mapsto \nabla \loss_\theta(x, y)
    : \theta \in B \}
    ~~~ \mbox{and} ~~~
    \mc{L}^f
    \defeq \{x \mapsto \nabla \loss_\theta(x, f(x)) : \theta \in B\}
  \end{equation*}
  are both $\P$-Donsker.
\end{assumption}

Let $Z \sim \normal(0, I)$. We will show that,
  for any $c < \infty$,
  \begin{equation}
  \label{eq:thm2_proofgoal}
    \limsup_n
    \sup_{\norm{h} \le c}
    \left|
    \P(\theta\opt + h/\sqrt{n} \in \CPP_\alpha)
    - \P\big(\ltwobig{-\Sigma^{-1/2} h + Z}^2
    \le \chi^2_{d,1-\alpha} \big) \right|
    = 0,
  \end{equation}
  for $\CPP_\alpha \in\{ \C_\alpha^{\mathrm{PP},T}, \C_\alpha^{\mathrm{PP},U}\}$, where $\Sigma$ is the asymptotic covariance from Theorem \ref{theorem:normality} for $\lambda = 1$, i.e. $\Sigma \equiv \Sigma^1 = H_{\theta\opt}^{-1} (r V^1_{f,\theta\opt} + V^1_{\Delta,\theta\opt}) H_{\theta\opt}^{-1}$.
  
The result \eqref{eq:thm2_proofgoal} directly implies the main result of the theorem.

  Let $h_n$ be any sequence such that $\sqrt{n} h_n \to h$ for some fixed $h
  \in \R^d$.  We give alternative representations of the statistics $T_n$
  and $U_n$ under the local alternative parameters $\theta\opt + h_n$. We
  leverage the following two technical lemmas:

 \begin{lemma}
    \label{lemma:variance-control}
    Under the conditions of Theorem~\ref{theorem:confidence-set-equivalence},
    \begin{equation*}
      \VhatDelta(\theta\opt + h_n)
      \cp \cov(\nabla \loss_{\theta\opt}
      - \nabla \loss_{\theta\opt}^f)
      ~~ \mbox{and} ~~
      \Vhatf(\theta\opt + h_n)
      \cp \cov(\nabla \loss_{\theta\opt}^f).
    \end{equation*}
  \end{lemma}

\begin{lemma}
    \label{lemma:gradient-process-linear}
   Define the (sub)gradient process
\begin{equation}
  \label{eqn:gradient-process}
\widehat G(\theta) \defeq \E_n[\nabla \loss_{\theta}
  - \nabla \loss_{\theta}^f] - \E[\nabla \loss_{\theta}
  - \nabla \loss_{\theta}^f] + \widetilde \E_N[\nabla \loss_{\theta}^f] - \E[\nabla \loss_{\theta}^f].
 \end{equation}
Then, under the conditions of Theorem~\ref{theorem:confidence-set-equivalence}, $\widehat G(\theta\opt + h_n) = \widehat G(\theta\opt) + o_P(1/\sqrt{n}).$
\end{lemma}

Lemma \ref{lemma:variance-control} follows by the classical
  result~\cite[Lemma 2.10.14]{VanDerVaartWe96} that if $\mc{F}$ is a Donsker
  class, then $\mc{F}^2 = \{f^2 : f \in \mc{F}\}$ is Glivenko-Cantelli.
  
  Lemma~\ref{lemma:gradient-process-linear} follows by the fact that the process is asymptotically stochastically equicontinuous, which follows by Assumption \ref{assumption:donsker}; see below for the proof.
  
\begin{proof}[Proof of Lemma \ref{lemma:gradient-process-linear}]
We adopt the definitions of $\GG_n g$, $\GG_N^f g$, and $\GG_n^f g$ from Equations \eqref{eq:Gn}, \eqref{eq:GNf}, and \eqref{eq:Gnf}, respectively.
    Because
    the process $[\GG_n \nabla \loss_\theta]_{\theta \in B}$
    converges to a (tight) Gaussian process in $\theta$
    by Assumption~\ref{assumption:donsker},
    as $h_n = O_P(1/\sqrt{n})$ we necessarily have
    \begin{equation*}
      \GG_n \nabla \loss_{\theta\opt + h_n}
      - \GG_n\nabla \loss_{\theta\opt}
      \cp 0
    \end{equation*}
    (as the process is asymptotically stochastically equicontinuous
    by Assumption~\ref{assumption:donsker}). Similarly,
    Assumption~\ref{assumption:donsker} implies $\GG_N^f \nabla \loss_{\theta\opt + h_n}
      - \GG_N^f\nabla \loss_{\theta\opt} \cp 0$ and $\GG_n^f \nabla \loss_{\theta\opt + h_n}
      - \GG_n^f\nabla \loss_{\theta\opt} \cp 0$.
    Adding and subtracting the preceding quantities and rescaling by $\sqrt{n}$,
    we obtain the lemma.
\end{proof}

Now we give asymptotically equivalent formulations of $T_n$ and $U_n$ evaluated at the alternatives $\theta\opt + h_n$. We can write $T_n$ as
\begin{equation*}
  T_n(\theta\opt + h_n)
  = \sqrt{n} \widehat\Sigma^{-1/2} (\thetaPP - \theta\opt - h_n)
  = \sqrt{n} \widehat\Sigma^{-1/2} (\thetaPP - \theta\opt)
  - \Sigma^{-1/2} h + o_P(1).
\end{equation*}
Turning to $U_n$, we let $V_{\Delta,\theta\opt} = \cov(\nabla \loss_{\theta\opt}
- \nabla \loss_{\theta\opt}^f)$ and
$V_{f,\theta\opt} = \cov(\nabla \loss_{\theta\opt}^f)$. Then Lemma \ref{lemma:variance-control} and
Lemma~\ref{lemma:gradient-process-linear} give
\begin{align*}
  &U_n(\theta\opt + h_n)\\
  &\quad \stackrel{(i)}{=} \sqrt{n} (V_{\Delta,\theta\opt} + \ratio V_{f,\theta\opt} + o_P(1))^{-1/2}
  \widehat G(\theta\opt + h_n)
  + \sqrt{n} (V_{\Delta,\theta\opt} + \ratio V_{f,\theta\opt} + o_P(1))^{-1/2}
  \nabla L(\theta\opt + h_n)
  \\
  &\quad \stackrel{(ii)}{=}
  \sqrt{n} (V_{\Delta,\theta\opt} + \ratio V_{f,\theta\opt})^{-1/2}
  \widehat G(\theta\opt) + (V_{\Delta,\theta\opt} + \ratio V_{f,\theta\opt})^{-1/2} H_{\theta\opt} h
  + o_P(1),
\end{align*}
where step~$(i)$ follows from Lemma~\ref{lemma:variance-control}
and $(ii)$ from Lemma~\ref{lemma:gradient-process-linear} coupled
with Slutsky's lemmas and
that $\nabla L(\theta\opt + h_n) = \nabla L(\theta\opt + h_n)
- \nabla L(\theta\opt) = \nabla^2 L(\theta\opt)h_n + o(\norm{h_n})
= H_{\theta\opt} h_n + o(1/\sqrt{n})$.

Rewriting the preceding two displays and applying Theorem \ref{theorem:normality} to the first,
we have
\begin{equation*}
  T_n(\theta\opt + h_n) \cd \normal(-\Sigma^{-1/2} h, I)
  ~~ \mbox{and} ~~
  U_n(\theta\opt + h_n) \cd \normal((V_{\Delta,\theta\opt} + \ratio V_{f,\theta\opt})^{-1/2} H_{\theta\opt} h, I).
\end{equation*}
The squared norm of the latter mean satisfies
\begin{equation*}
  \ltwo{(V_{\Delta,\theta\opt} + \ratio V_{f,\theta\opt})^{-1/2} H_{\theta\opt} h}^2
  = h^T H_{\theta\opt} (V_{\Delta,\theta\opt} + \ratio V_{f,\theta\opt})^{-1}
  H_{\theta\opt} h =
  h^T \Sigma^{-1} h
\end{equation*}
by the definition of $\Sigma$. Therefore, the rotational invariance of the Gaussian yields
\begin{equation}
  \begin{split}
    \P(\theta\opt + h_n \in \C_\alpha^{\mathrm{PP},T})
    & = \P(\ltwo{T_n(\theta\opt + h_n)}^2
    \le \chi^2_{d,1-\alpha})
    = \P\left(\ltwo{-\Sigma^{-1/2} h + Z}^2 \le
    \chi^2_{d,1-\alpha}\right) + o(1) \\
    \P(\theta\opt + h_n \in \C_\alpha^{\mathrm{PP},U})
    & = \P(\ltwo{U_n(\theta\opt + h_n)}^2
    \le \chi^2_{d,1-\alpha})
    = \P\left(\ltwo{-\Sigma^{-1/2} h + Z}^2 \le
    \chi^2_{d,1-\alpha}\right) + o(1),
  \end{split}
  \label{eqn:almost-home}
\end{equation}
where $Z \sim \normal(0, I)$.

Finally, we employ a standard compactification argument. Assume for the sake
of contradiction that
\begin{equation*}
  \limsup_n \sup_{\norm{h} \le c}
  \left|\P(\theta\opt + h / \sqrt{n} \in \CPP_\alpha)
  - \P\left(\ltwo{-\Sigma^{-1/2} h + Z}^2 \le \chi^2_{d,1-\alpha}\right)
  \right| > 0,
\end{equation*}
for one of $\CPP_\alpha \in \{\C_\alpha^{\mathrm{PP},T}, \C_\alpha^{\mathrm{PP},U}\}$.
If this is the case, then there must be a bounded subsequence of $h_n$
with
\begin{equation*}
  \left|\P(\theta\opt + h_n / \sqrt{n} \in \CPP_\alpha)
  - \P\left(\ltwo{-\Sigma^{-1/2} h_n + Z}^2 \le \chi^2_{d,1-\alpha}\right)
  \right| \to a > 0.
\end{equation*}
Take a convergent subsequence from this set, so that $\sqrt{n} (h_n / \sqrt{n})
= h_n \to h$. The existence of this subsequence contradicts
the limits~\eqref{eqn:almost-home}. This proves the first claim of the theorem.

For the second claim, take any $t_n = \omega(1/\sqrt{n})$ and $\delta\in(0,1)$. Then
\begin{align*}
  \P\left( \theta\opt + h t_n \in \C_\alpha^{\mathrm{PP},T}\right) &\leq \P\left( \|\sqrt{n}\widehat\Sigma^{-1/2}  h t_n \| \leq \chi_{d,1-\alpha} + \|\sqrt{n}\widehat\Sigma^{-1/2}  (\thetaPP - \theta\opt) \|\right)\\
  &\leq \P\left( \|\sqrt{n}\widehat\Sigma^{-1/2}  h t_n \| \leq \chi_{d,1-\alpha} + \chi_{d,1-\delta}\right) + \P\left( \|\sqrt{n}\widehat\Sigma^{-1/2}  (\thetaPP - \theta\opt) \| > \chi_{d,1-\delta}\right).
\end{align*}
Taking the limit, we get
\begin{equation*}
  \limsup_n \sup_{\|h\|\geq c} \P\left( \theta\opt + h t_n \in \C_\alpha^{\mathrm{PP},T}\right) \leq  \delta.
\end{equation*}
As $\delta$ was arbitrary, we conclude that $\limsup_n \sup_{\|h\|\geq c}
\P( \theta\opt + h t_n \in \C_\alpha^{\mathrm{PP},T}) = 0$.

\subsection{Proof of Proposition \ref{prop:optimal-lambda}}
\label{sec:proof-optimal-lambda}

By definition
\begin{equation*}
  \lambda\opt = \argmin_{\lambda} \Tr\left(\Sigma^\lambda\right) = \argmin_{\lambda} \Tr\left(H_{\theta\opt}^{-1} \left(r \cdot \lambda^2 \cov(\nabla \loss_{\theta\opt}^f) + \cov( \nabla \loss_{\theta\opt} -  \lambda \nabla \loss_{\theta\opt}^f ) \right) H_{\theta\opt}^{-1}\right).
\end{equation*}
Writing $\cov( \nabla \loss_{\theta\opt} - \lambda \nabla
\loss_{\theta\opt}^f ) = \cov(\nabla \loss_{\theta\opt}) + \lambda^2
\cov(\nabla \loss_{\theta\opt}^f) - \lambda (\Cov(\nabla
\loss_{\theta\opt}, \nabla \loss_{\theta\opt}^f) + \Cov(\nabla
\loss_{\theta\opt}^f, \nabla \loss_{\theta\opt} ))$ and applying the
linearity of the trace, $\lambda\opt$ evidently minimizes
\begin{align*}
\lambda^2(1+r) \Tr\left(H_{\theta\opt}^{-1} \cov(\nabla \loss_{\theta\opt}^f)H_{\theta\opt}^{-1}\right) - \lambda \Tr\left(H_{\theta\opt}^{-1}\left(\Cov\left(\nabla \loss_{\theta\opt}, \nabla \loss_{\theta\opt}^f \right) + \Cov\left(\nabla \loss_{\theta\opt}^f, \nabla \loss_{\theta\opt} \right) \right)H_{\theta\opt}^{-1} \right).
\end{align*}
Analytically optimizing this quadratic gives the optimal choice
\begin{equation*}
  \lambda\opt = \frac{\Tr\left(H_{\theta\opt}^{-1}\left(\Cov\left(\nabla \loss_{\theta\opt}, \nabla \loss_{\theta\opt}^f \right) + \Cov\left(\nabla \loss_{\theta\opt}^f, \nabla \loss_{\theta\opt} \right) \right)H_{\theta\opt}^{-1} \right)}{2(1+r) \Tr\left(H_{\theta\opt}^{-1}\cov(\nabla \loss_{\theta\opt}^f) H_{\theta\opt}^{-1} \right) }.
\end{equation*}

%% file: one-step-experiments.tex

\section{Experiments with one-step PPI}
\label{app:one-step-ppi}

We repeat the synthetic experiments from the main text
(Section~\ref{sec:simulations}), but only compare PPI (with $\lambda=1$), power-tuned PPI
(\ppipp) with clipping $\hat \lambda$ to $[0,1]$, and one-step PPI.  The figures below show the results in sequence;
one-step PPI and power-tuned PPI overlap almost completely in all plots but
one: Figure~\ref{fig:onestep-ppi-mean-anticorrelated}.  This figure shows
results from a new experiment using anticorrelated predictions (the same
setting as in Figure~\ref{fig:mean_estimation}, except that we take $f(X_i)
= -Y_i + \epsilon$).  In this setting, one-step PPI outperforms PPI and
power-tuned PPI because it allows $\hat\lambda$ to be negative.

\begin{figure}[ht]
    \centering
    \includegraphics[width=0.8\textwidth]{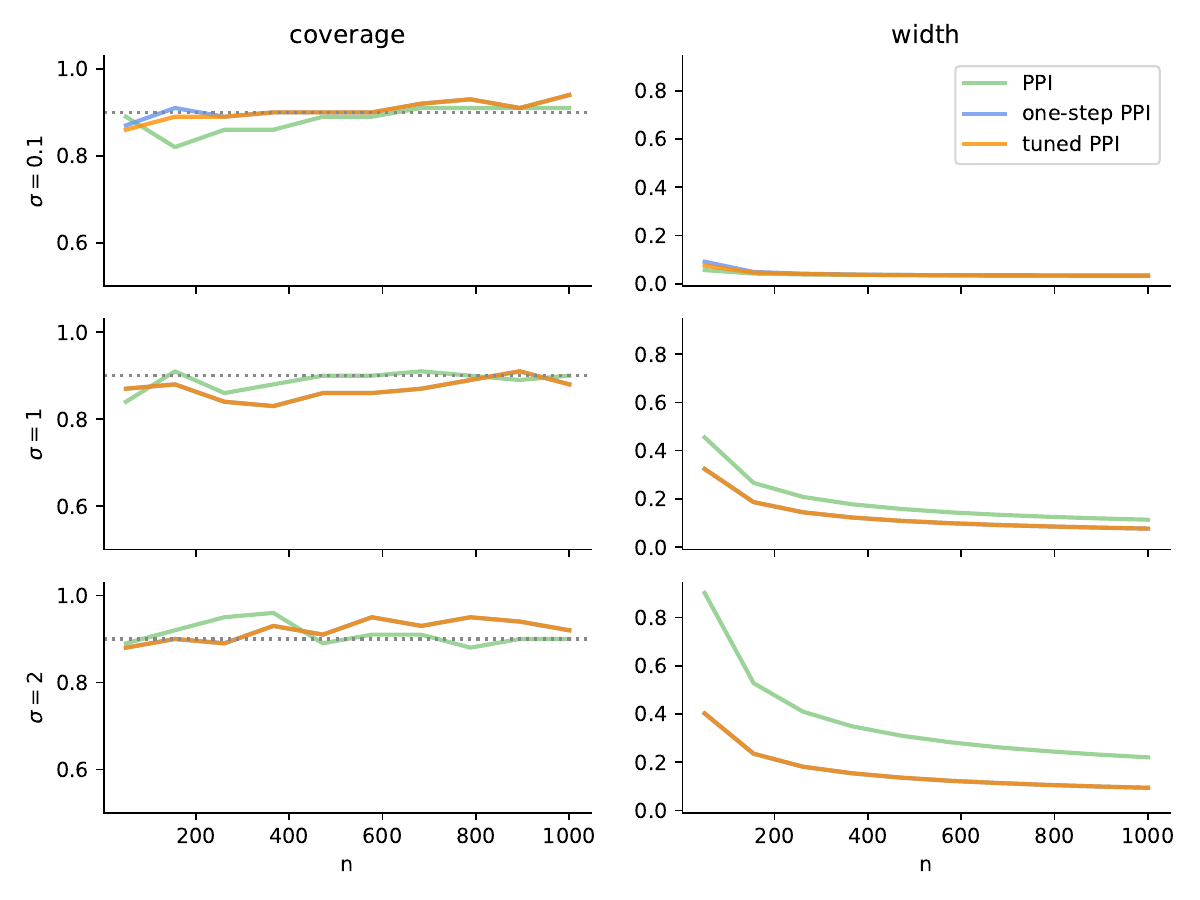}
    \caption{Mean estimation with one-step PPI. The left panel shows coverage and the right shows width. The setup is the same as in Figure \ref{fig:mean_estimation}.}
    \label{fig:onestep-ppi-mean}
\end{figure}

\begin{figure}[ht]
    \centering
    \includegraphics[width=0.8\textwidth]{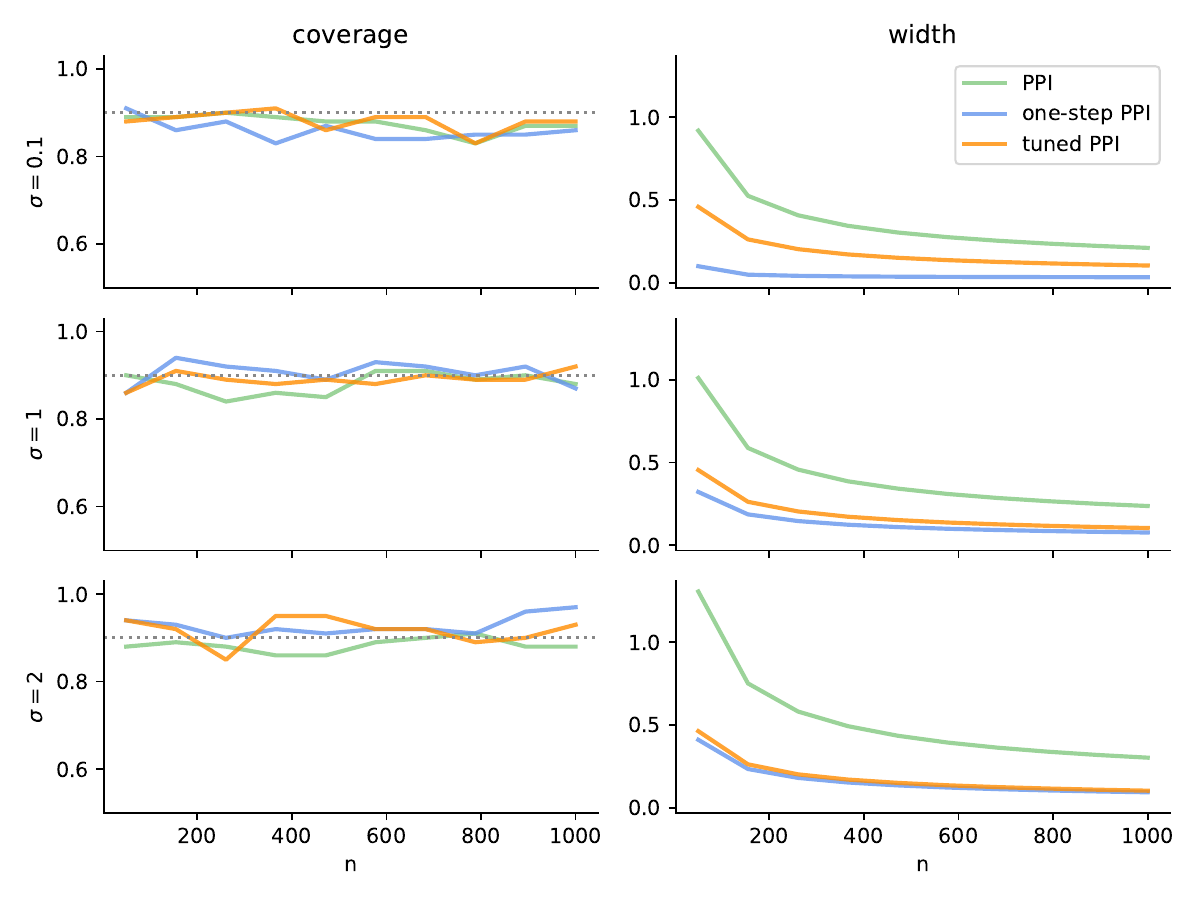}
    \caption{Mean estimation with one-step PPI and anticorrelated predictions. The left panel shows coverage and the right shows width. The setup is the same as in Figure \ref{fig:mean_estimation}, only we multiply the true outcomes by $-1$ when generating the predictions, $f(X_i) = -Y_i + \epsilon$.}
    \label{fig:onestep-ppi-mean-anticorrelated}
\end{figure}

\begin{figure}[ht]
    \centering
    \includegraphics[width=0.8\textwidth]{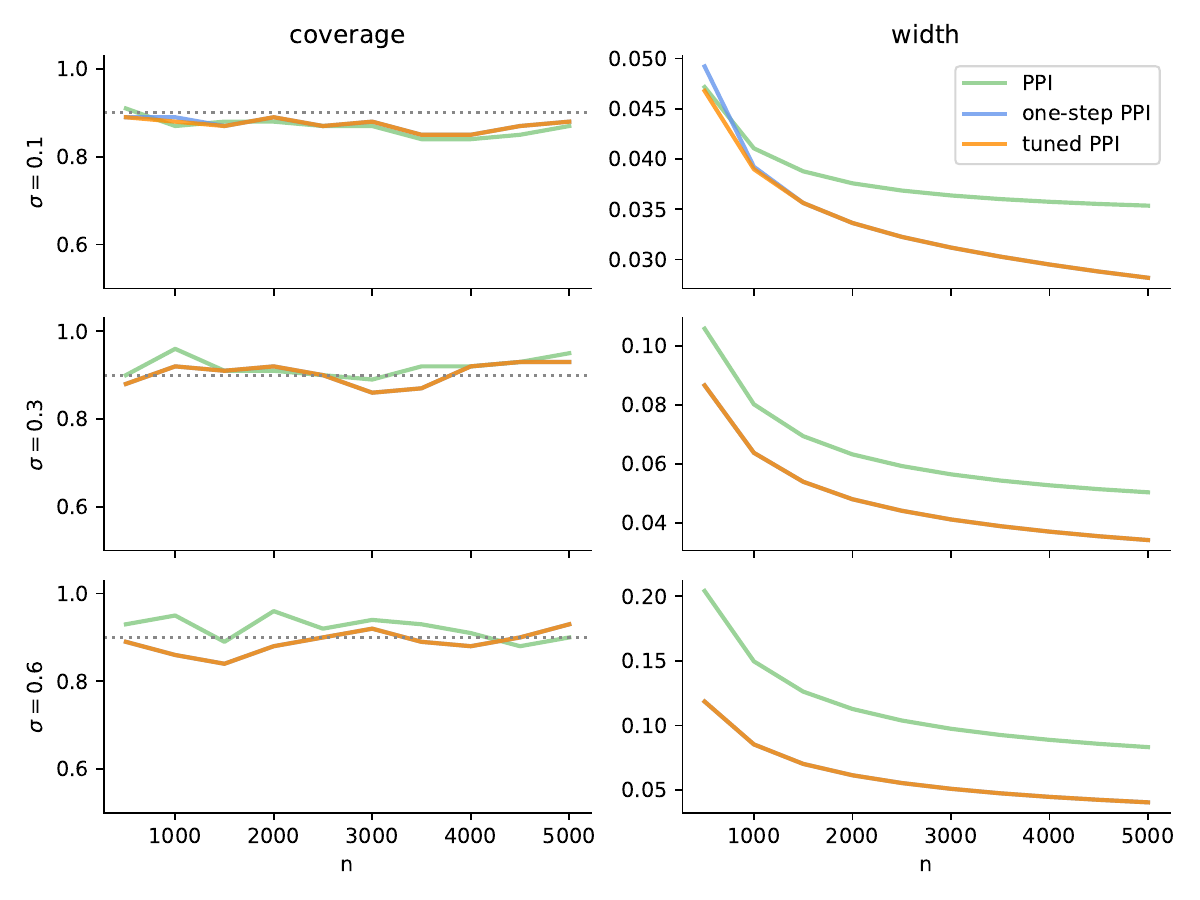}
    \caption{Linear regression with one-step PPI. The left panel shows coverage and the right shows width. The setup is the same as in Figure \ref{fig:linear_regression}.}
    \label{fig:onestep-ppi-linear}
\end{figure}
\begin{figure}[ht]
    \centering
    \includegraphics[width=0.8\textwidth]{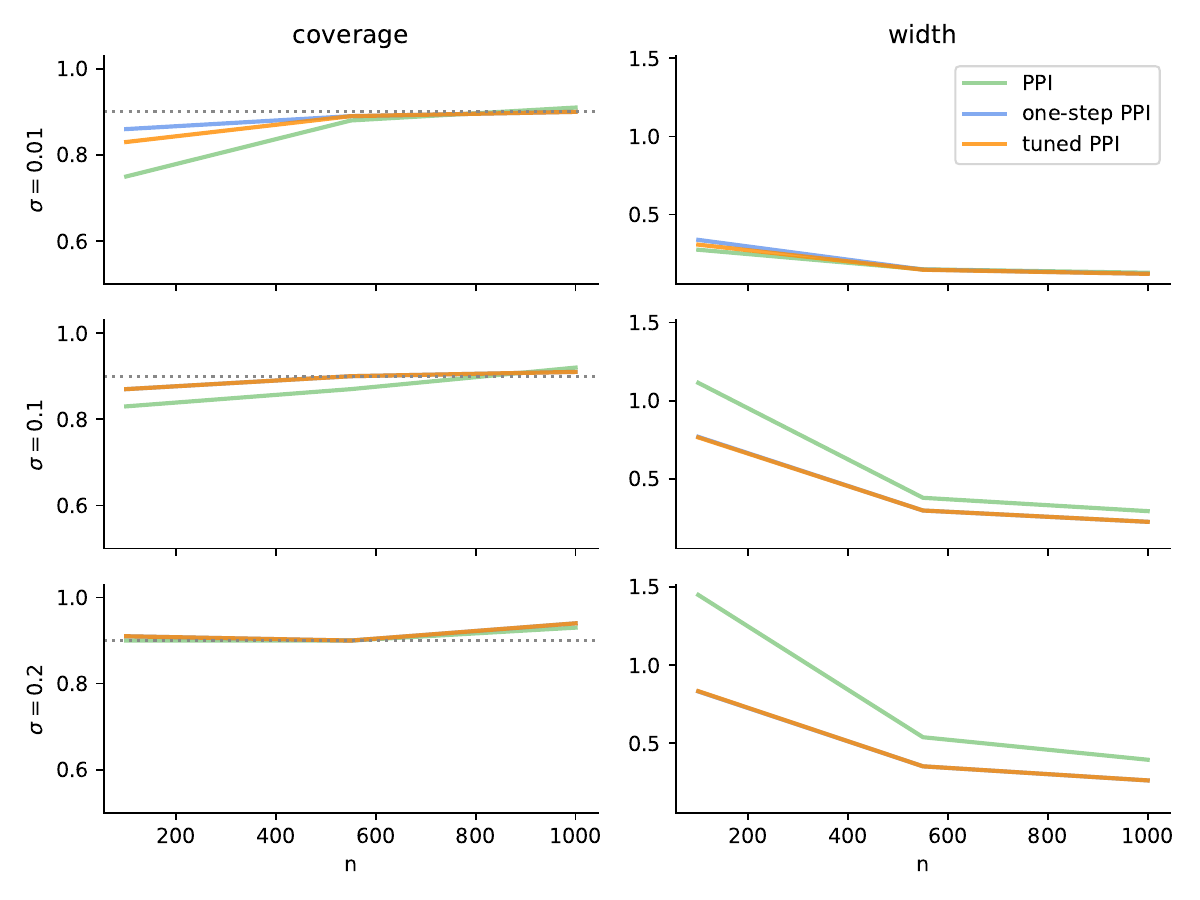}
    \caption{Logistic regression with one-step PPI. The left panel shows coverage and the right shows width. The setup is the same as in Figure \ref{fig:logistic_regression}.}
    \label{fig:onestep-ppi-logistic}
\end{figure}

%% file: odds-ratio-calculations.tex
\section{Confidence intervals on the odds ratio}
\label{app:odds-ratio}

\newcommand{\tZ}{\widetilde{Z}}
\newcommand{\diag}{\textup{diag}}

We provide an asymptotically exact prediction-powered confidence interval on
the odds ratio, which allows for power tuning. We are interested in
quantifying the strength of association between two events $Z\in\{0,1\}$ and
$Y\in\{0,1\}$, defining the odds ratio
\begin{equation}
  \label{eq:odds_ratio}
  \frac{\mu_1/(1-\mu_1)}{\mu_0/(1-\mu_0)},
\end{equation}
where $\mu_1 = \P(Y = 1 \mid Z=1)$ and $\mu_0 = \P(Y = 1 \mid Z=0)$.  We
consider the setting where the labels $Y$ are expensive or difficult to
measure, while the indicator $Z$ is readily available. We therefore have a
small dataset of size $n$ with both $Z_i$ and $Y_i$ as well as other
covariates $X_i$ and a large dataset of size $N$ with only $\tZ_i$ and
$\tX_i$ on which $f(\tX_i)$ imputes the missing labels.
We construct confidence intervals for the log odds ratio
$\theta\opt = \log\frac{\mu_1/(1-\mu_1)}{\mu_0/(1-\mu_0)}$, which yields
a confidence interval the odds
ratio~\eqref{eq:odds_ratio} by exponentiating.

We condition on the indicator variables $\{Z_i\}_{i=1}^n$ and
$\{\tZ_i\}_{i=1}^N$.
For a bit $b \in \{0, 1\}$, denote
$n_b = |\{i \in [n] \mid Z_i = b\}|$
and $N_b = |\{i \in [N] \mid \tZ_i = b\}|$,
so that the natural estimators of the means
$\mu_0$ and $\mu_1$ are
\begin{align*}
  \hat{\mu}_b^\lambda
  & \defeq
  \frac{1}{n_b} \sum_{i=1}^n \ind{Y_i = 1, Z_i = b} +
  \lambda \left(\frac{1}{N_b} \sum_{i=1}^N \ind{f(\Xt_i)=1, \tilde Z_i = b} -
  \frac{1}{n_b} \sum_{i=1}^n \ind{f(X_i)=1,  Z_i = b} \right) \\
  & \,=
  P_n(Y = 1 \mid Z = b) + \lambda\left(\wt{P}_n(f = 1 \mid Z = b)
  - P_n(f = 1 \mid Z = b)\right),
\end{align*}
where $P_n$ and $\wt{P}_N$ denote the empirical distributions
on the labeled and unlabeled samples.
Let
$r_b = \lim \frac{n_b}{N_b}$
and define the
conditional variances
\begin{equation*}
  \sigma_b^2(\lambda)
  \defeq \var(Y \mid Z = b)
  - 2 \lambda \cov(Y, f(X) \mid Z = b)
  + \lambda^2 (1 + r_b) \var(f(X) \mid Z = b),
\end{equation*}
indexed by $b \in \{0, 1\}$ and $\lambda \in \R$.
Then
Theorem~\ref{theorem:normality} shows that
\begin{align*}
  \sqrt{n}_b (\hat{\mu}_b^\lambda - \mu_b)
  \cd \normal\left(0, \sigma^2_b(\lambda)\right)
\end{align*}
conditionally almost surely (conditionally on the sequences
$Z_i$ and $\tZ_i$).
Because $\hat{\mu}_0^{\lambda_0}$ and $\hat{\mu}_1^{\lambda_1}$ are
independent conditional on $\{Z_i\}_{i=1}^n$ and $\{\tZ_i\}_{i=1}^N$,
we obtain (a.s.) the joint convergence
\begin{equation*}
  \sqrt{n}\left((\hat{\mu}^{\lambda_0}_0 , \hat{\mu}^{\lambda_1}_1)
  - (\mu_0, \mu_1)
  \right)
  \cd \normal\left(\zeros,
  \diag(\sigma_0^2(\lambda_0), \sigma_1^2(\lambda_1))\right).
\end{equation*}
Applying the delta method
to $\phi(t) = \log\frac{t}{1 - t}$ (with $\phi'(t) = \frac{1}{t(1 - t)}$),
for any choice
of $\lambda_0, \lambda_1$, the natural estimator
$\hat{\theta} =
\log \frac{\hat{\mu}_1^{\lambda_1}}{1 - \hat{\mu}_1^{\lambda_1}}
- \log \frac{\hat{\mu}_0^{\lambda_0}}{1 - \hat{\mu}_0^{\lambda_0}}$
satisfies
\begin{align}
  \nonumber
  \lefteqn{\sqrt{n} (\hat{\theta} - \theta\opt)} \\
  & = \sqrt{\frac{n}{n_1}} \sqrt{n_1}
  \left(
  \log \frac{\hat{\mu}_1^{\lambda_1}}{1 - \hat{\mu}_1^{\lambda_1}}
  - \log \frac{\mu_1}{1 - \mu_1}\right)
  - \sqrt{\frac{n}{n_0}} \sqrt{n_0}
  \left(
  \log \frac{\hat{\mu}_0^{\lambda_0}}{1 - \hat{\mu}_0^{\lambda_0}}
  - \log \frac{\mu_0}{1 - \mu_0}\right) \nonumber \\
  & \cd \normal\left(0, \frac{1}{p_1 \mu_1^2(1 - \mu_1)^2}
  \sigma^2_1(\lambda_1)
  + \frac{1}{p_0 \mu_0^2 (1 - \mu_0)^2} \sigma^2_0(\lambda_0)
  \right),
  \label{eqn:limit-log-odds-ratio}
\end{align}
where $p_b = \lim \frac{n_b}{N} = \P(Z = b)$ for $b \in \{0, 1\}$.


Finally, we apply Proposition~\ref{prop:optimal-lambda}
to obtain the natural plug-in variance-minimizing estimates for the
weights $\lambda$, which gives
\begin{equation*}
  \hat{\lambda}_b =
  \frac{\widehat\Cov_n(Y_i, f(X_i) \mid Z_i = b)}{(1 + \frac {n_b} {N_b})
    \widehat \var_{N+n}(f(X_i) \mid Z_i = b)}.
\end{equation*}
Standard continuous mapping arguments~\cite{VanDerVaart98}
then imply that the \ppipp estimator
\begin{equation*}
  \thetaPP
  = \log \frac{\hat{\mu}_1^{\hat{\lambda}_1}}{1 - \hat{\mu}_1^{\hat{\lambda}_1}}
  - \log \frac{\hat{\mu}_0^{\hat{\lambda}_0}}{1 - \hat{\mu}_0^{\hat{\lambda}_0}},
\end{equation*}
when we substitute in the asymptotic~\eqref{eqn:limit-log-odds-ratio},
satisfies
\begin{equation*}
  \sqrt{n}(\thetaPP - \theta\opt)
  \cd \normal\left(0, \frac{1}{p_1 \mu_1^2(1 - \mu_1)^2} \inf_\lambda
  \sigma_1^2(\lambda)
  + \frac{1}{p_0 \mu_0^2(1 - \mu_0)^2} \inf_\lambda
  \sigma_0^2(\lambda)\right).
\end{equation*}
The CLT implies a confidence interval for the odds ratio by substituting
empirical estimates of the variance and covariance terms in the asymptotic
variance.